\documentclass[10pt,twocolumn,twoside]{IEEEtran}
\usepackage{amsfonts}
\usepackage{mathrsfs}
\usepackage{amsfonts}
\usepackage{amssymb}
\usepackage{graphicx}
\usepackage{amsmath}
\usepackage{array}
\usepackage{cases}

\usepackage{cite,graphicx,amsmath,amssymb,color}

\usepackage{multirow}

\begin{document}

\title{Distributive Stochastic Learning for Delay-Optimal OFDMA Power and Subband Allocation}

\newtheorem{Thm}{Theorem}
\newtheorem{Lem}{Lemma}
\newtheorem{Cor}{Corollary}
\newtheorem{Def}{Definition}
\newtheorem{Exam}{Example}
\newtheorem{Alg}{Algorithm}
\newtheorem{Prob}{Problem}
\newtheorem{Rem}{Remark}
\newtheorem{Proof}{Proof}
\newtheorem{Asump}{Assumption}
\newtheorem{Prop}{Proposition}

\author{Ying~Cui, \IEEEmembership{Student~Member,~IEEE}, Vincent K.~N.~Lau, \IEEEmembership{Senior~Member,~IEEE},
\\
Department of Electronic and Computer Engineering\\
The Hong Kong University of Science and Technology\\
{\em cuiying@ust.hk}, {\em eeknlau@ee.ust.hk}\\
\thanks{This work is funded by RGC 615609.}}
\markboth{To appear in Transactions on Signal
Processing}{}\maketitle

%
%
%
%
%
%
%
%
%

\begin{abstract}
In this paper, we consider the distributive queue-aware power and
subband allocation design for a delay-optimal OFDMA uplink system
with one base station, $K$ users and $N_F$ independent subbands.
Each mobile has an uplink queue with heterogeneous packet arrivals
and delay requirements. We model the problem as an infinite horizon
average reward Markov Decision Problem (MDP) where the control
actions are functions of the instantaneous Channel State Information
(CSI) as well as the joint Queue State Information (QSI).  To
address the distributive requirement and the issue of exponential
memory requirement and computational complexity, we approximate the
subband allocation Q-factor by the sum of the per-user subband
allocation Q-factor and derive a distributive online stochastic
learning algorithm to estimate the per-user Q-factor and the
Lagrange multipliers (LM) simultaneously and determine the control
actions using an auction mechanism. We show that under the proposed
auction mechanism, the distributive online learning converges almost
surely (with probability 1).
For illustration, we apply the proposed distributive stochastic
learning framework to an application example with exponential packet
size distribution. We show that the delay-optimal power control has
the {\em multi-level water-filling} structure where the CSI
determines the instantaneous power allocation and the QSI determines
the water-level. The proposed algorithm has linear signaling
overhead and computational complexity $\mathcal O(KN)$, which is
desirable from an implementation perspective.

\end{abstract}

\section{Introduction}\label{sec_intro}
There are plenty of literature on cross-layer optimization of power
and subband allocation in
OFDMA systems \cite{ofdma:Cioffi2006,ofdma:Roger1999}. 
Yet, all these works  focused on optimizing the physical layer
performance and the power/subband allocation solutions derived are
all functions of the channel state information (CSI) only. On the
other hand, real life applications are delay-sensitive and it is
critical to consider the bursty arrivals and delay performance in
addition to the conventional physical layer performance (such as
sum-rate or proportional fair) in OFDMA cross-layer design. A
combined framework taking into account of both queueing delay and
physical layer performance is not trivial as it involves both the
queueing theory (to model the queue dynamics) and information theory
(to model the physical layer dynamics).  The first approach converts
the delay constraint into average rate constraint using tail
probability at large delay regime and solve the optimization problem
using information theoretical formulation based on the rate
constraint\cite{Hui:2007,Tang:2007}. While this approach allows
potentially simple solution, the derived control policy will be a
function of the CSI only, which is good only for large delay regime
where the probability of buffer empty is small.  In general,
delay-optimal control actions should be a function of both the CSI
and queue state information (QSI).  In \cite{Yeh:01PhD}, the authors
showed that the Longest Queue Highest Possible Rate (LQHPR) policy
is delay optimal for multiaccess fading channels. However, the
solution utilizes stochastic majorization theory which requires
symmetry among the users and is difficult to extend to other
situations. In \cite{Neelybook:2006,stability:1999}, the authors
studied the {\em queue stability region} of various wireless systems
using Lyapunov drift, which is good only for large delay region.
While all the above works addressed different aspects of the delay
sensitive resource allocation problem, there are still a number of
first order issues to be addressed to obtain decentralized resource
optimization for delay-optimal uplink OFDMA system.

\begin{itemize}

\item {\bf The Curse of Dimensionality} A more general approach is to model the problem as a {\em Markov Decision Problem} (MDP) \cite{Bertsekas:1987,Cao:2007}. However, a primary  difficulty in determining the optimal policy using the MDP approach is
the huge state space involved\footnote{As illustrated later, we
could derive the closed form action given the Q-factor. Therefore,
the curse of the dimension refers to the exponential growth of state
space only. The dimensionality of the action space is not an
issue.}. For instance, the state space is exponentially large in the
number of users. Consider a simple example. For a system with 4
users, $6$ independent subbands, a buffer size of 50 per user and 4
channel states, the system state space contains
$4^{4\times6}\times(50+1)^4$ states, which is already unmanageable.


\item {\bf Decentralized Solution} Most of the solutions in the literature are centralized in which the
processing is done at the base station\cite{WongOFDMAresource}
requiring global knowledge of CSI and QSI from $K$ users.  However,
in the uplink direction, the QSI is only available locally at each
of the $K$ users. Hence, centralized solution at the BS requires all
the $K$ users to deliver their QSI to the BS, which consumes
enormous signaling overhead, and the BS to broadcast the allocation
results for the resource allocations at the Mobile side in the
uplink system. In addition, the centralized solution also leads to a
exponential computational complexity to the BS.

\item {\bf Convergence of Stochastic Iterative Solution}
There are a number of works on decentralized OFDMA control  using
deterministic game \cite{RayLiu:2007} or primal-dual decomposition
theory for solving deterministic NUM
\cite{PalomarMungdecomposition:2006}. The derived distributive
algorithms are iterative in nature where all the nodes exchange some
messages explicitly in solving the {\em master problem}. However,
the CSI is always assumed to be quasi-static during the iterative
updates with message passing. When we consider delay-optimization,
the problem is stochastic in nature and is quite challenging because
the game is played repeatedly and the actions as well as the payoffs
are defined over ergodic realizations of the system states (CSI,
QSI). During the iterative updates, the system state will not be
quasi-static anymore.
\end{itemize}

In this paper, we consider an OFDMA uplink system with one base
station (BS), $K$ users and $N_F$ independent subbands. The
delay-optimal problem is cast into an infinite horizon average
reward constrained Markov Decision Process (MDP).  To address the
distributive requirement and the issue of exponential memory
requirement and computational complexity, we propose a distributive
online stochastic learning algorithm, which only requires knowledge
of the local QSI and the local CSI at each of the $K$ mobiles and
determine the resource control actions using a per-stage auction.
Using separation of time scales, we show that under the proposed
auction mechanism, the distributive online learning converges almost
surely.
For illustration, we apply the proposed distributive stochastic
learning framework to an application example with exponential packet
size distribution. We show that the delay-optimal power control has
the {\em multi-level water-filling} structure where the CSI
determines the instantaneous power allocation and the QSI determines
the water-level. We show that the proposed algorithm converges to
the global optimal solution for sufficiently large number of users.
The proposed algorithm has linear signaling overhead and
computational complexity $\mathcal O(KN)$, which is desirable from
an implementation perspective.


\begin{figure}[h]
\begin{center}
\includegraphics[height=6cm, width=8cm]{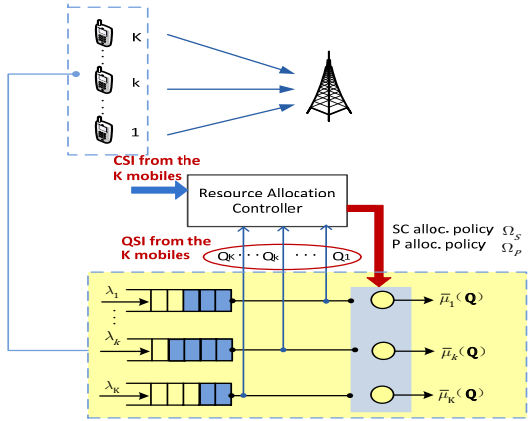}
\caption{OFDMA physical layer and queueing model.}
\label{systemmodel}
\end{center}
\end{figure}

\section{System Models}\label{sec_sys_mod}
\label{sect2}

In this section, we shall elaborate the system model, the OFDMA
physical layer model as well as the underlying queueing model. There
are one BS and $K$ mobile users (each with one uplink queue) in the
OFDMA uplink system with $L$ subcarriers over a frequency selective
fading channel with $N_F$ independent multipaths as illustrated in
Figure~\ref{systemmodel}.  The BS has a {\em cross-layer controller}
which takes the joint channel state information (CSI) and joint
queue state information (QSI) as the inputs and produces a power
allocation and subband allocation actions as outputs\footnote{We
first formulate the problem in a centralized manner and address the
distributive solution in Section \ref{sec_decentralized solution}.}.

We first list the important notations in this paper in table
\ref{table:notation}.
\begin{table}[h]
\begin{center} {\footnotesize
\begin{tabular}{|c|c|}
\hline
$K$ & number of users     \\
$N_F$ & number of independent subbands\\
$N_Q$ &   buffer size  \\
$k$, $n$ & user, subband index \\
$\overline{N}_k$ & mean packet size of user $k$  \\
$t$ & slot index \\
$s_{k,n}$, $p_{k,n}$ & subband, power allocation action \\
$\Omega=(\Omega_p, \Omega_s)$ & power and subband
allocation policy \\
$\mathbf{H}=\{|H_{k,n}|\}$ & joint CSI \\
$\mathbf{Q}=(Q_k)$ & joint QSI  \\
$\mathbf{A}=(A_k)$ & bit/packet arrival vector \\
$\boldsymbol{\chi}=\big(\mathbf{H},\mathbf{Q}\big)$ &
global system state  \\
$\tau$ & frame duration\\
$\lambda_k$ & average arrival rate of user $k$\\
$\mu_k(\boldsymbol{\chi})$ & conditional mean departure rate of user
$k$\\
$P_k$, $P_k^d$ & total power, packet drop rate of user $k$\\
$\{V(\boldsymbol{\chi})\}$ & system potential function on
$\boldsymbol{\chi}$\\
$\{\mathcal{Q}(\boldsymbol{\chi},\mathbf{s})\}$ &  subband
allocation Q-factor\\
$\{\mathcal{Q}^k(\boldsymbol{\chi}_k,\mathbf{s}_k)\}$ & per-user
subband allocation Q-factor \\
$\{q^k(Q, H, s)\}$ & per-user per-subband subband allocation
Q-factor\\
$\overline{\gamma}^k$ & LM w.r.t. average power constraint of
$k$\\
$\underline{\gamma}^k$& LM w.r.t. average pck drop constraint of
$k$\\
$\{\epsilon_t^{q}\}$ & stepsize sequence for per-user potential
update\\
$\{\epsilon_t^{\gamma}\}$ & stepsize sequence for per-user 2 LMs
update\\
 \hline
\end{tabular}}
\end{center}
\caption{\footnotesize List of Important Notations.}
\label{table:notation}
\end{table}

\subsection{OFDMA Physical Layer Model}\label{subsec_mimo_phy_mod}
Let $s_{k,n}\in\{0,1\}$ denote the subband allocation for the $k$-th
user at the $n$-th subband. The received signal from the $k$-th user
at the $n$-th subband of the base station is given by $Y_{k,n}^r =
s_{k,n}(H_{k,n} X_{k,n}^t + Z_{k,n})$, where $X_{k,n}^t$ is the
transmitted symbol, $H_{k,n}$ and $Z_{k,n}(\sim \mathcal{CN}(0,1))$
are the random fading  and the channel noise of the $k$-th user at
the $n$-th subband respectively. The data rate of  user $k$ can be
expressed as:
\begin{equation}
R_k =\sum_{n=1}^{N_F} R_{k,n}=\sum_{n=1}^{N_F} s_{k,n}\log\left(1+
\xi p_{k,n}|H_{k,n}|^2\right) \label{eqn:rate1}
\end{equation}
for some constant $\xi$. Note that the data rate expression in
\eqref{eqn:rate1}  can be used to model both the uncoded and coded
systems. For uncoded system using MQAM constellation, the BER of the
$n$-th subband and the $k$-th user is given by\cite{RayLiu:2007}
$BER_{k,n} \approx  c_1
\exp(-c_2\frac{\Gamma_{k,n}}{2^{R_{k,n}}-1})$, where $\Gamma_{k,n}$
is the received SNR of the $k$-th user at the $n$-th subband, and
hence, for a target BER $\epsilon$, we have $\xi = -\frac{c_2}{\ln
(\epsilon/c_1)}$.  On the other hand, for system with powerful error
correction codes such as LDPC with reasonably large block length
(e.g 8Kbyte) and target PER of 0.1$\%$, the maximum achievable data
rate is given by the instantaneous mutual information (to within 0.5dB SNR). 
In that case, $\xi = 1$\footnote{In this paper, our derived results
are based on $\xi = 1$ for notation simplicity, which can be easily
extended to other cases.}.

\subsection{Source Model, Queue Dynamics and Control Policy}
\label{subsec_queue_mod} In this paper, the time dimension is
partitioned into {\em scheduling slots} indexed by $t$ with slot
duration $\tau$.
\begin{Asump}
The joint CSI of the system is denoted by
$\mathbf{H}(t)=\{|H_{k,n}(t)|\forall k,n\}$, where $|H_{k,n}(t)|$ is
a discrete r.v. distributed according to $\Pr[|H|]$. The CSI is
quasi-static within a scheduling slot and i.i.d. between scheduling
slots\footnote{The quasi-static assumption is realistic for
pedestrian mobility users where the channel coherence time is around
50 ms but typical frame duration is less than 5ms in next generation
wireless systems such as WiMAX. On the other hand, we assume the CSI
is i.i.d. between slots (as in many other literature) in order to
capture first order insights. Similar solution framework can also be
extended to deal with correlated fading.}. ~ \hfill\QED
\label{Asump:H}
\end{Asump}
Let $\mathbf{A}(t)=\big(A_1(t),\cdots,A_K(t)\big)$ be the random new
arrivals (number of bits) at the end of the $t$-th scheduling slot.
\begin{Asump}
The arrival process $A_k(t)$ is i.i.d. over scheduling slots
according to a general distribution $\Pr(A_k)$ with average arrival
rate $\mathbb{E}[A_k]=\lambda_k$. ~ \hfill\QED
\label{Asump:A}
\end{Asump}

Let $\mathbf{Q}(t)=\big(Q_1(t),\cdots,Q_K(t)\big)$ be the joint QSI
of the $K$-user OFDMA system, where $Q_k(t)$ denotes the number of
bits in the $k$-th queue at the beginning of the $t$-th slot. $N_Q$
denotes the maximum buffer size (number of bits). Thus, the
cardinality of the joint QSI is $I_Q=(N_Q+1)^K$, which grows
exponentially with $K$. 
Let $N_H$ denote the cardinality of $|H_{k,n}|$ ($\forall k,n$).
Hence, the cardinality of the global CSI is given by
$I_H=N_H^{N_FK}$. Let $\mathbf{R}(t)=\big(R_1(t),\cdots,R_K(t)\big)$
(bits/second) be the scheduled data rates  of the $K$ users, where
$R_k(t)$ is given by (\ref{eqn:rate1}). We assume the controller is
causal so that new bit arrivals $\mathbf{A}(t)$ is observed only
after the controller's actions at the t-th slot. Hence, the queue
dynamics is given by the following equation:
\begin{align}
Q_k(t + 1) = \min \Big \{\big[Q_k(t)-R_k(t)\tau\big]^+ + A_k(t),N_Q
\Big\}, \nonumber\\
 \forall k \in
\{1,K\}\label{eqn:general-queue-evolution-bit}
\end{align}
where $x^+\triangleq\max\{x,0\}$ and $\tau$ is the  duration of a
scheduling slot.

For notation convenience, we denote
$\boldsymbol{\chi}(t)=\big(\mathbf{H}(t),\mathbf{Q}(t)\big)$ to be
the {\em global system state} at the $t$-th slot. Therefore, the
cardinality of the state space of $\boldsymbol{\chi}$ is
$I_{\chi}=I_H\times I_Q=\big(N_H^{N_F}(N_Q+1)\big)^K$. Given the
observed system state realization $\boldsymbol{\chi}(t)$ at the
beginning of the $t$-th slot, the transmitter may adjust the
transmit power and subband allocation (equivalently data rate
$\mathbf{R}(t)$) according to a {\em stationary power control and
subband allocation policy} defined below\footnote{At the beginning
of the $t$-th scheduling slot, the cross-layer controller observes
the joint CSI $\mathbf{H}(t)$ and the joint QSI $\mathbf{Q}(t)$ and
determines the transmit power and subband allocation  across the $K$
users.}.

\begin{Def}(\emph{Stationary Power Control and Subband Allocation Policy})
A stationary transmit power and subband allocation policy
$\Omega=(\Omega_p,\Omega_s)$ is a mapping from the system state
$\boldsymbol{\chi}$ to the power and subband allocation actions. A
policy $\Omega$ is called {\em feasible} if the associated actions
satisfy the average total transmit power constraint and the subband
assignment constraint. Specifically,
$\Omega_p(\boldsymbol{\chi})=\mathbf{p}=\{p_{k,n}\geq 0: \forall
k,n\}$ and
$\Omega_s(\boldsymbol{\chi})=\mathbf{s}=\{s_{k,n}\in\{0,1\}: \forall
k,n\}$ satisfy
\begin{align}
&\sum_{n=1}^{N_F} \mathbb{E}[p_{k,n}] \leq P_k,\ \forall
k\in\{1,K\},
\label{eqn:tx-pwr-constraint1}\\
&\sum_{k=1}^K s_{k,n} = 1,\ \forall n\in\{1,N_F\}
\label{eqn:subband-constraint1}
\end{align}
Furthermore, $\Omega$ also satisfies an average packet drop rate
constraint for each queue as follows
\begin{equation}
\Pr[Q_k=N_Q]\leq P^d_k, \ \forall k\in\{1,K\}
\label{eqn:ptdrop-constraint1}
\end{equation}~ \hfill\QED
\label{Defn:1}
\end{Def}

From \eqref{eqn:rate1}, the vector queue dynamics is Markovian with
the transition probability given by
\begin{align}
& \Pr[\mathbf{Q}(t+1)|
\boldsymbol{\chi}(t),\Omega(\boldsymbol{\chi}(t))]\nonumber\\
=&
\Pr\big[\mathbf{A}(t)
=\mathbf{Q}(t+1)-[\mathbf{Q}(t)-\mathbf{R}(t)\tau]^+\big]
\nonumber\\
=&\prod_k \Pr\big[A_k(t)=Q_k(t+1)-[Q_k(t)-R_k(t)\tau]^+\big]
\label{general-vec-q-trans-prob}
\end{align}
Note that the $K$ queues are coupled together via the control policy
$\Omega$ and the constraint in (\ref{eqn:subband-constraint1}). 
From Assumption \ref{Asump:H}, the induced random process
$\boldsymbol{\chi}(t)=(\mathbf{H}(t),\mathbf{Q}(t))$ is Markovian
with the following transition probability\footnote{Although the QSI
$\mathbf{Q}(t+1)$ and CSI $\mathbf{H}(t)$ are correlated via the
control action $\Omega(\boldsymbol{\chi}(t))$, due to the i.i.d.
assumption of CSI in Assumption \ref{Asump:H}, $\mathbf{H}(t+1)$ is
independent of $\boldsymbol{\chi}(t)$. Note that $\mathbf H (t)$
being i.i.d. is a special case of Markovian model. Hence,
\eqref{eqn:transition-prob1} holds under the $\mathbf H (t)$ i.i.d.
assumption in Assumption \ref{Asump:H}.}
\begin{align}
 &
\Pr[\boldsymbol{\chi}(t+1)|\boldsymbol{\chi}(t),\Omega(\boldsymbol{\chi}(t))]\nonumber\\
=&
\Pr[\mathbf{H}(t+1)|\boldsymbol{\chi}(t),\Omega(\boldsymbol{\chi}(t))]
\Pr[\mathbf{Q}(t+1)|\boldsymbol{\chi}(t),\Omega(\boldsymbol{\chi}(t))]\nonumber\\
=&\Pr[\mathbf{H}(t+1)]
\Pr[\mathbf{Q}(t+1)|\boldsymbol{\chi}(t),\Omega(\boldsymbol{\chi}(t))]
\label{eqn:transition-prob1}
\end{align}
where
$\Pr[\mathbf{Q}(t+1)|\boldsymbol{\chi}(t),\Omega(\boldsymbol{\chi}(t))]$
is given by \eqref{general-vec-q-trans-prob}. Given a unichain
policy $\Omega$, the induced Markov chain $\{\boldsymbol{\chi}(t)\}$
is ergodic and there exists a unique steady state distribution
$\pi_{\chi}$ where
$\pi_{\chi}(\boldsymbol{\chi})=\lim_{t\rightarrow\infty}\Pr[\boldsymbol{\chi}(t)
= \boldsymbol{\chi}]$. The average utility of the $k$-th user under
a unichain policy $\Omega$ is given by:
\begin{align}
\overline{T}_k(\Omega) = \lim_{T\rightarrow\infty}
\frac{1}{T}\sum_{t=1}^T\mathbb{E}[f(Q_k(t))]
=\mathbb{E}_{\pi_{\chi}}\left[f(Q_{k})\right],\nonumber\\
 \forall
k\in\{1,K\} \label{eqn:delay1}
\end{align}
where $f(Q_{k})$ is a monotonic increasing function of $Q_k$ and
$\mathbb{E}_{\pi_{\chi}}$ denotes expectation w.r.t. the underlying
measure $\pi_{\chi}$. For example, when
$f(Q_k)=\frac{Q_k}{\lambda_k}$, $\overline{T}_k(\Omega)
=\frac{1}{\lambda_k} \mathbb{E}_{\pi_{\chi}}\left[Q_{k}\right]$ is
the {\em average delay} of the $k$-th user. Another interesting
example is the {\em queue outage probability}
$\overline{T}_k(\Omega)=\Pr[Q_k\geq Q_k^o]$, in which
$f(Q_k)=\mathbf{1}[Q_k\geq Q_k^o]$, where $Q_k^o \in \{0,N_Q\}$ is
the reference outage queue state.
Similarly, the average transmit power constraint in
(\ref{eqn:tx-pwr-constraint1}) and the packet drop constraint in
\eqref{eqn:ptdrop-constraint1} can be written as
\begin{align}
\overline{P_{k}}(\Omega)=
\lim_{T\rightarrow\infty}\frac{1}{T}\sum_{t=1}^T\mathbb{E}\Big[\sum_{n}p_{k,n}(t)\Big]
= &\mathbb{E}_{\pi_{\chi}}\Big[\sum_{n}p_{k,n}\Big] \leq
P_k,\nonumber\\
& \ \forall k\in\{1,K\}
 \label{eqn:tx-pwr2}
\end{align}
\begin{align}
\overline{P_{k}^d}(\Omega)=&
\lim_{T\rightarrow\infty}\frac{1}{T}\sum_{t=1}^T\mathbb{E}\Big[\mathbf{1}[Q_k(t)=N_Q]\Big]\nonumber\\
= &\mathbb{E}_{\pi_{\chi}}\Big[\mathbf{1}[Q_k=N_Q]\Big] \leq P_k^d,\
\forall k\in\{1,K\}
 \label{eqn:ptdrop2}
\end{align} 


\section{CMDP Formulation and General Solution of the Delay-Optimal Problem }\label{sec_MDP_formulation}
In this section, we shall formulate the delay-optimal problem as  an
infinite horizon average reward constrained Markov Decision Problem
(CMDP) and discuss the general solution.

\subsection{CMDP Formulation}
\label{subsec_mdp_formulation}

A MDP can be characterized by a tuple of four objects, namely the
state space, the action space, the transition probability kernel as
well as the per-stage reward function. In our delay-optimization
problem, we could associate these four objects as follows:
\begin{itemize}
\item \textbf{State Space}: The state space for the MDP is given by $\{\boldsymbol{\chi}^1,\cdots,
\boldsymbol{\chi}^{I_{\chi}}\}$, where
$\boldsymbol{\chi}^i=\big(\mathbf{H}^i,\mathbf{Q}^i\big)$ ($1\leq i
\leq I_{\chi}$) is a realization of the global system state.
\item \textbf{Action Space}: The action space of the MDP is given by $\{\Omega(\boldsymbol{\chi}^1),\cdots,
\Omega(\boldsymbol{\chi}^{I_\chi})\}$, where $\Omega$ is a unichain
feasible policy as defined in Definition \ref{Defn:1}.
\item \textbf{Transition Kernel}: The transition kernel of the MDP $\Pr[\boldsymbol{\chi}^j|\boldsymbol{\chi}^i,\Omega(\boldsymbol{\chi}^i)]$ is given by \eqref{eqn:transition-prob1}.
\item \textbf{Per-stage Reward}: The per-stage reward function of the MDP is given by $d(\boldsymbol{\chi},\Omega(\boldsymbol{\chi}))=\sum_{k}
\beta_k f(Q_k)$.
\end{itemize}

As a result, the delay-optimal control can be formulated as a CMDP,
which is summarized below.
\begin{Prob}[Delay-Optimal Constrained MDP]
For some positive constants
$\boldsymbol{\beta}=(\beta_1,\cdots,\beta_K)$\footnote{The positive
weighting factors $\boldsymbol{\beta}$ in \eqref{cons-MDP} indicate
the relative importance of buffer delay among the $K$ data streams
and for each given $\boldsymbol{\beta}$, the solution to
\eqref{cons-MDP} corresponds to a point on the Pareto optimal delay
tradeoff boundary of a {\em multi-objective} optimization problem.},
the delay-optimal problem is formulated as
\begin{align}
&
\min_{\Omega}J_{\beta}(\Omega)=\sum_{k=1}^K\beta_k\overline{T}_k(\Omega)\nonumber\\
& \quad \quad \quad \quad \quad = \lim_{T\rightarrow\infty}
\frac{1}{T}\sum_{t=1}^T \mathbb{E}\Big[
d(\boldsymbol{\chi}(t),\Omega(\boldsymbol{\chi}(t)))\Big]
\label{cons-MDP}\\
&  \text{s.t.} \ \text{the power and packet drop rate constraints
in} \,(\ref{eqn:tx-pwr2}), \eqref{eqn:ptdrop2} \nonumber
\end{align}
\label{Prob2}
\end{Prob}


\subsection{Lagrangian Approach to the CMDP}\label{subsec_embedded_mar_chain}

For any LMs $\overline{\gamma}^k, \underline{\gamma}^k>0$, define
the Lagrangian as
$L_{\beta}(\Omega,\boldsymbol{\gamma})=\lim_{T\rightarrow\infty}
\frac{1}{T}\sum_{t=1}^T \mathbb{E}\big[
g(\boldsymbol{\gamma},\boldsymbol{\chi},\Omega(\boldsymbol{\chi}))\big]$,
where $\boldsymbol{\gamma}=(\gamma^1,\cdots,\gamma^K)$ with
$\gamma^k=(\overline{\gamma}^k,\underline{\gamma}^k$),
$g(\boldsymbol{\gamma},\boldsymbol{\chi},\Omega(\boldsymbol{\chi}))
= \sum_{k} \Big(\beta_k f(Q_k)+\overline{\gamma}^k (\sum_n
p_{k,n}-P_k)+\underline{\gamma}^k(\mathbf{1}[Q_k=N_Q]-P_k^d)\Big)$.
Thus, the corresponding unconstrained MDP for a particular LM
$\boldsymbol{\gamma}$ is given by
\begin{align}
G(\boldsymbol{\gamma})=\min_{\Omega}L_{\beta}(\Omega,\boldsymbol{\gamma})
\label{uncons-MDP}
\end{align}
where $G(\boldsymbol{\gamma})$ gives the Lagrange dual function. The
dual problem of the primal problem in Problem \ref{Prob2} is given
by $\max_{\boldsymbol{\gamma}\succeq0}G(\boldsymbol{\gamma})$. The
general solution to the unconstrained MDP in \eqref{uncons-MDP} is
summarized in the following lemma.
\begin{Lem}(\emph{Bellman Equation and Subband Allocation Q-factor})
For a given $\boldsymbol{\gamma}$, the optimizing policy for the
unconstrained MDP in \eqref{uncons-MDP} can be obtained by solving
the {\em Bellman equation}  (associated with the MDP in
\eqref{cons-MDP} ) w.r.t.
$(\theta,\{\mathcal{Q}(\boldsymbol{\chi},\mathbf{s})\})$ as below:
\begin{align}
\mathcal{Q}(\boldsymbol{\chi}^i,\mathbf{s})=&
\min_{\Omega_p(\boldsymbol{\chi}^i)} \Big[
g(\boldsymbol{\gamma},\boldsymbol{\chi}^i,\mathbf{s},\Omega_p(\boldsymbol{\chi}^i))\nonumber\\
& + \sum_{\boldsymbol{\chi}^j}
\Pr[\boldsymbol{\chi}^j|\boldsymbol{\chi}^i,\mathbf{s},\Omega_p(\boldsymbol{\chi}^i)]\min_{\mathbf{s}'}
\mathcal{Q}(\boldsymbol{\chi}^j,\mathbf{s}')\Big]-\theta \nonumber\\
& \quad \quad \quad \quad \quad \quad \quad \quad \quad \quad
\forall \,1\leq i\leq I_{\chi}, \ \forall \mathbf{s}
\label{eqn:Bellman-q-factor}
\end{align}
where $\theta = L_{\beta}^*(\boldsymbol{\gamma}) =
\min_{\Omega}L_{\beta}(\Omega,\boldsymbol{\gamma})$ is the optimal
average reward per stage and
$\{\mathcal{Q}(\boldsymbol{\chi},\mathbf{s})\}$ is the {\em subband
allocation Q-factor}. The optimal control policy\footnote{It is
known that for CDMP, the optimal policy may be randomized policy.
However, for implementation consideration, we are interested in
deterministic policy in this paper.} is  given by
$\Omega^*=(\Omega_p^*,\Omega_s^*)$ with
$\Omega_p^*(\boldsymbol{\chi}^i)$ attaining the minimum of the
R.H.S. of \eqref{eqn:Bellman-q-factor} and
$\Omega_s^*(\boldsymbol{\chi}^i)=\arg
\min_{\mathbf{s}}\mathcal{Q}(\boldsymbol{\chi}^i,\mathbf{s})$ for
any $\boldsymbol{\chi}^i$. Since the policy space we considered
consists of only unichain policies, the associated Markov chain
$\{\boldsymbol{\chi}(t)\}$ is irreducible and there exists a
recurrent state\footnote{For sufficiently large total transmit power
$\{P_1,\cdots,P_K\}$ so that the optimization problem in
\eqref{cons-MDP} is feasible, and the state
$\boldsymbol{\chi}=(\mathbf{H},\mathbf{Q})$ ($\forall \mathbf{H}$
and $\mathbf{Q}=(0,\cdots,0)$) is recurrent.}. Hence, the solution
to \eqref{eqn:Bellman1} is unique up to an additive
constant\cite{Bertsekas:2007}.~ \hfill\QED
\label{Lem:bellman-eqn-Q-factor}
\end{Lem}
\begin{proof}
Please refer to Appendix A for the proof.
\end{proof}

Using standard optimization theory\cite{boydconvex:2004}, the
problem in \eqref{uncons-MDP} has an optimal solution for a
particular choice of the LM
$\boldsymbol{\gamma}=\boldsymbol{\gamma}^*$, where
$\boldsymbol{\gamma}^*$ is chosen to satisfy the average power
constraint in \eqref{eqn:tx-pwr2} and packet drop constraint in
\eqref{eqn:ptdrop2}. Moreover, it is shown in
\cite{Borkaractorcritic:2005} that the following saddle point
condition holds:
\begin{eqnarray}
L(\Omega^*, \boldsymbol{\gamma})\leq L(\Omega^*,
\boldsymbol{\gamma}^*) \leq L(\Omega, \boldsymbol{\gamma}^*)
\label{saddle-pt-cond}
\end{eqnarray}
In other words, $(\Omega^*,\boldsymbol{\gamma}^*)$ is a saddle point
of the Lagrangian, then $\Omega^*$ is the primal optimal (i.e.
solving Problem \ref{Prob2}), $\boldsymbol{\gamma}^*$ is the dual
optimal (solving the dual problem) and the duality gap is zero.
Therefore, by solving the dual problem, we can obtain the primal
optimal $\Omega^*$.

\begin{Rem}
The optimal control actions are functions of the {\em subband
allocation Q-factor} $\{\mathcal{Q}(\boldsymbol{\chi},\mathbf{s})\}$
and the $2K$ LMs. Unfortunately, for any given LMs, determining the
subband allocation Q-factor involves solving the Bellman equation in
\eqref{eqn:Bellman-q-factor}, which is a fixed point problem over
the functional space with exponential complexity. In other words, it
is a system of $K^{N_F}I_{\chi}=K^{N_F}\big(N_H^{N_F}(N_Q+1)\big)^K$
non-linear equations with $K^{N_F}I_{\chi}+1$ unknowns
$(\theta,\{\mathcal{Q}(\boldsymbol{\chi},\mathbf{s})\})$. 
Furthermore, even if we could solve it, the solution will be
centralized and the joint  CSI and  QSI knowledge will be required,
which is highly undesirable.~ \hfill\QED
\end{Rem}

\section{General Decentralized Solution via Localized Stochastic Learning and Auction} \label{sec_decentralized solution}
The key steps in obtaining the optimal control policies from the
R.H.S. of the Bellman equation in \eqref{eqn:Bellman-q-factor} rely
on the knowledge of the subband allocation Q-factor
$\{\mathcal{Q}(\boldsymbol{\chi},\mathbf{s})\}$ and the $2K$ LMs
$\{\overline{\gamma}^k,\underline{\gamma}^k\}$ ($1 \leq k\leq K$),
which is very challenging. Brute-force solution of
$\{\mathcal{Q}(\boldsymbol{\chi},\mathbf{s})\}$ and $2K$ LMs has
exponential complexity and requires centralized implementation and
knowledge of the joint CSI and QSI (which also requires huge
signaling
overheads). In this section, 
we shall approximate the {\em subband allocation Q-factor}
$\mathcal{Q}(\boldsymbol{\chi},\mathbf{s})$ by the sum of {\em
per-user subband allocation Q-factor}
$\mathcal{Q}^k(\boldsymbol{\chi}_k,\mathbf{s}_k)$, i.e.
$\mathcal{Q}(\boldsymbol{\chi},\mathbf{s})\thickapprox \sum_k
\mathcal{Q}^k(\boldsymbol{\chi}_k,\mathbf{s}_k)$. 
Based on the approximate Q-factor, we shall derive a per-stage
decentralized control policy using a {\em per-stage auction}. Next,
we shall propose a localized online stochastic learning algorithm
(performed locally at each MS $k$) to determine the {\em per-user
Q-factor} $\{\mathcal{Q}^k(\boldsymbol{\chi}_k,\mathbf{s}_k)\}$ as
well as the two local LMs $\gamma^k=(\overline{\gamma}^k,
\underline{\gamma}^k)$ based on observations of the local CSI and
local QSI as well as the
auction result. 
Furthermore, we shall prove that under the proposed per-stage
auction, the local online stochastic learning algorithm converges
almost surely  (with probability 1).


\subsection{Linear Approximation on the Subband Allocation Q-Factor and Distributive Power Control}\label{subsec_Q_factor_approx}

Denote the per-user system state, channel state, subband allocation
actions and  power control actions as $\boldsymbol{\chi}_k=(Q_k,
\mathbf{H}_k)$, $\mathbf{H}_k=\{|H_{k,n}|: \forall n\}$,
$\mathbf{s}_k=\{s_{k,n}:\forall n\}$ and
$\mathbf{p}_k=\{p_{k,n}:\forall n\}$, respectively. To reduce the
size of the state space and to decentralize the resource allocation,
we approximate $\mathcal{Q}(\boldsymbol{\chi},\mathbf{s})$ by the
sum of {\em per-user subband allocation Q-factor}
$\mathcal{Q}^k(\boldsymbol{\chi}_k,\mathbf{s}_k)$, i.e.
\begin{align}
\mathcal{Q}(\boldsymbol{\chi},\mathbf{s})\thickapprox \sum_k
\mathcal{Q}^k(\boldsymbol{\chi}_k,\mathbf{s}_k)\label{eqn:approximate-Q}
\end{align}
where $\mathcal{Q}^k(\boldsymbol{\chi}_k,\mathbf{s}_k)$ satisfies
the following {\em per-user subband allocation Q-factor fixed point
equation} for each MS $k$:
\begin{align}
\mathcal{Q}^k(\boldsymbol{\chi}^i_k,\mathbf{s}_k)=
\min_{\mathbf{p}_k} \Big[&
g_k(\boldsymbol{\gamma}^k,\boldsymbol{\chi}^i_k,\mathbf{s}_k,\mathbf{p}_k)\nonumber\\
&+ \sum_{\boldsymbol{\chi}^j_k}
\Pr[\boldsymbol{\chi}^j_k|\boldsymbol{\chi}^i_k,\mathbf{s}_k,\mathbf{p}_k]
W^k(\boldsymbol{\chi}^j_k)\Big]-\theta^k, \nonumber\\
&\quad \quad \quad \quad \quad \quad \quad\forall 1\leq i\leq
I_{\chi}^k,  \forall \mathbf{s}_k
\label{eqn:Bellman-per-user-q-factor}
\end{align}
where
$g_k(\boldsymbol{\gamma}^k,\boldsymbol{\chi}_k,\mathbf{s}_k,\mathbf{p}_k)=\beta_k
f(Q_k)+\overline{\gamma}^k (\sum_n
p_{k,n}-P_k)+\underline{\gamma}^k(\mathbf{1}[Q_k=N_Q]-P_k^d)$ and
$W^k(\boldsymbol{\chi}_k)=\mathbb{E}\big[
\mathcal{Q}^k\big(\boldsymbol{\chi}_k,\{s_{k,n}=\mathbf{1}[|H_{k,n}|\geq
H^*_{K-1}]\}\big)|\boldsymbol{\chi}_k \big] $ ($H^*_{K-1}$ denotes
the largest order statistic of the $(K-1)$ i.i.d. random variables
with the same distribution as $|H_{k,n}|$), and
$I^k_{\chi}=N_H^{N_F}(N_Q+1)$ is the cardinality of  the space of
per-user system state. Note that under the subband allocation
Q-factor approximation, the state space of $K$ users is
significantly reduced from $I_{\chi}=\big(N_H^{N_F}(N_Q+1)\big)^K$
to $KI^k_{\chi}=KN_H^{N_F}(N_Q+1)$.

\subsection{Per-Stage Subband Auction}\label{subsec_auction}

The subband allocation control can be obtained by minimizing the
original subband allocation Q-factor in \eqref{eqn:Bellman-q-factor}
over subband allocation actions. Using the approximate Q-factor, the
subband allocation control is given by
$\Omega_s^*(\boldsymbol{\chi})=\arg
\min_{\mathbf{s}}\mathcal{Q}(\boldsymbol{\chi},\mathbf{s})\approx
\arg \min_{\mathbf{s}} \sum_k
\mathcal{Q}^k(\boldsymbol{\chi}_k,\mathbf{s}_k)$. This can be
obtained via a per-stage subband auction with $K$ bidders (MSs) and
one auctioneer (BS) based on the observed realization of the system
state at each MS $\boldsymbol{\chi}_k$. The {\em Per-Stage Subband
Auction} among $K$ MSs is as follows:

\begin{itemize}
\item \textbf{Bidding}:  Based on the local observation $\boldsymbol{\chi}_k$, each user $k$ submits his bid $\{\mathcal{Q}^k(\boldsymbol{\chi}_k,\mathbf{s}_k):\forall
\mathbf{s}_k\}$.

\item \textbf{Subband Allocation}: The BS assigns subbands
to achieve the maximum sum bids, i.e.
\begin{align}
\mathbf{s}^*=\Omega_s^*(\boldsymbol{\chi})=\arg \min_{\mathbf{s}}
\sum_k \mathcal{Q}^k(\boldsymbol{\chi}_k,\mathbf{s}_k)
\label{eqn:general-s}
\end{align}
and then broadcasts the allocation results
$\mathbf{s}^*=\{\mathbf{s}^*_k:\forall k\}$ to $K$ users.
\item \textbf{Power Allocation}: Based on the subband allocation result $\mathbf{s}^*_k$, each user $k$ determines the transmit
power, which minimizes the R.H.S. of
\eqref{eqn:Bellman-per-user-q-factor}, i.e.
\begin{align}
&\mathbf{p}_k^*=\Omega_{p_k}^*(\boldsymbol \chi)
g_k(\boldsymbol{\gamma}^k,\boldsymbol{\chi}^i_k,\mathbf{s}_k^*,\mathbf{p}_k)\label{eqn:general-p}\\
=&\arg \min_{\mathbf{p}_k} \Big[+ \sum_{\boldsymbol{\chi}^j_k}
\Pr[\boldsymbol{\chi}^j_k|\boldsymbol{\chi}^i_k,\mathbf{s}_k^*,\mathbf{p}_k]
W^k(\boldsymbol{\chi}^j_k)\Big]-\theta^k\nonumber
\end{align}
\end{itemize}

\begin{Rem}(\emph{Optimal Subband and Power Allocation under Q-factor
Approximation}) In proposed per-stage subband auction, the subband
allocation actions minimize $\sum_k
\mathcal{Q}^k(\boldsymbol{\chi}_k,\mathbf{s}_k)$, and the power
allocation actions at each MS minimizes the R.H.S. of the per-user
subband allocation Q-factor fixed point equation in
\eqref{eqn:Bellman-per-user-q-factor}. Therefore, the proposed
per-stage subband auction achieves the solution of the Bellman
equation in \eqref{eqn:Bellman-q-factor} under the linear Q-factor
approximation in \eqref{eqn:approximate-Q}. ~ \hfill\QED
\end{Rem}

\begin{Rem} (\emph{Computational Complexity and Memory Requirement Reduction at
BS}) With the per-stage subband auction mechanism, the BS does not
need to store the per-user subband allocation Q-factor
$\{\mathcal{Q}^k(\boldsymbol{\chi}_k,\mathbf{s}_k)\}$ ($\forall k$)
and $2K$ LMs for all the MSs, which greatly reduced the memory
requirement at the BS. On the other hand, the BS does not need to
perform power allocation for each MS on each subband $p_{k,n}$
($\forall k,n$), which significantly reduces the computational
complexity at the BS. ~ \hfill\QED
\end{Rem}

\subsection{Online Per-user Primal-Dual  Learning Algorithm via Stochastic Approximation}\label{subsec_online distri_alg}

Since the derived power and subband allocation policies are all
functions of the per-user subband allocation Q-factor and LMs, we
shall propose an online localized learning algorithm to estimate
$\{\mathcal{Q}^k(\boldsymbol{\chi}_k,\mathbf{s}_k)\}$ and LMs
$\gamma^k$ at each MS $k$. For notation convenience, we denote the
per-user state-action combination as
$\varphi\triangleq(\boldsymbol{\chi}_k,\mathbf{s}_k)$ ($\forall k$).
Let $i$ and $j$ ($1\leq i,j \leq I_{\varphi}$) be the dummy indices
enumerating all the per-user state-action combinations of each user
 with cardinality $I_{\varphi}=2^{N_F}
I^k_{\chi}$.  Let $\boldsymbol{\mathcal{Q}}^k\triangleq
\big(\mathcal{Q}^k(\varphi^1),\cdots,
\mathcal{Q}^k(\varphi^{I_{\varphi}})\big)^T$ be the vector of
per-user Q-factor for user k. Let $\varphi_k
(t)\triangleq\big(\boldsymbol{\chi}_k(t),\mathbf{s}_k(t)\big)$ be
the state-action pair observed at MS $k$ at the $t$-th slot, where
$\boldsymbol{\chi}_k(t)=(Q_k(t), \mathbf{H}_k(t))$ is the system
state realization observed at MS $k$. Based on the current
observation $\varphi_k (t)$, user $k$ updates its estimate on the
per-user Q-factor and the LMs according to:
\begin{align}
&\ \mathcal Q^k_{t+1}(\varphi^i)=
       \mathcal Q^k_t(\varphi^i) + \epsilon_{l_k(\varphi^i,t)}^{q}
\big[g_k(\boldsymbol{\gamma}^k_t,\varphi^i,\mathbf p_k(t))\nonumber \\
&\quad \quad \quad \quad \quad \quad
+\widetilde{W}^k_t(Q_k(t+1))\big)
-\big(g_k(\boldsymbol{\gamma}^k_t,\varphi^r,\mathbf p_k(\bar{t}))\nonumber\\
&\quad \quad \quad \quad \quad \quad+ \widetilde{W}^k_t(Q_k(\bar
t+1)) -\mathcal Q^k_t(\varphi^r)\big) -
 \mathcal Q^k_t(\varphi^i) \big]\nonumber\\
& \quad \quad \quad \quad \quad \quad \cdot \mathbf{1}\big[\varphi_k
(t)=\varphi^i \big]
 \label{eqn:Q-est1}\\
&\overline{\gamma}_{t+1}^k=\Gamma\big(\overline{\gamma}^k_t+\epsilon_t^{\gamma}(\sum_{n}p_{k,n}(t)-P_k)\big)
 \label{eqn:gamma-dis-est1}\\
&\underline{\gamma}_{t+1}^k=\Gamma\big(\underline{\gamma}^k_t+\epsilon_t^{\gamma}(\mathbf{1}[Q_k(t)=N_Q]-P_k^d)\big)
 \label{eqn:gamma-dis-est1-under}
\end{align}
where $l_k(\varphi^i,t)\triangleq\sum_{m=0}^t
\mathbf{1}\big[\varphi_k (m)=\varphi^i\big]$ is the number of
updates of $\mathcal Q^k(\varphi^i)$ till
$t$\cite{Borkarasynchronous:1998}, $\mathbf
p_k(t)=\{p_{k,n}(t):\forall n\}$ is the power allocation actions
given be the per-stage auction, $\widetilde{W}^k_t(Q_k)\triangleq
\mathbb{E}[W^k_t(\boldsymbol{\chi}_k)|Q_k]$ with
$W^k_t(\boldsymbol{\chi}_k)=\mathbb{E}\big[
\mathcal{Q}^k_t\big(\boldsymbol{\chi}_k,\{s_{k,n}=\mathbf{1}[|H_{k,n}|\geq
H^*_{K-1}]\}\big)|\boldsymbol{\chi}_k \big] $, $\bar{t}\triangleq
\sup\{t:\varphi_k(t)=\varphi^r\}$, $\varphi^r$ is the reference
per-user state-action combination\footnote{Without loss of
generality, we initialize  the per-user subband allocation Q-factor
as 0, i.e. $\mathcal Q^k_0(\varphi^r)=0\ \forall k$.},
$\Gamma(\cdot)$ is the projection onto an interval $[0,B]$ for some
$B>0$ and $\{\epsilon_t^{q}\},\{\epsilon_t^{\gamma}\}$ are the step
size sequences satisfying the following conditions:
\begin{align}
\sum_t\epsilon_t^{q}=\infty, &\epsilon_t^{q}\geq 0, \epsilon_t^{q}
\rightarrow 0 ,\sum_t\epsilon_t^{\gamma}=\infty,
\epsilon_t^{\gamma}\geq 0,\epsilon_t^{\gamma} \rightarrow 0,\nonumber\\
&\sum_t\big(({\epsilon_t^{q}})^2+2({\epsilon_t^{\gamma}})^2 \big) <
\infty, \frac{\epsilon_t^{\gamma}}{\epsilon_t^{q}}\rightarrow 0
\label{stepsize}
\end{align}

The above distributive per-user potential learning algorithm
requires knowledge on local QSI and local CSI only.

\begin{Rem}[Comparison to the Deterministic NUM]
In conventional iterative solutions for deterministic
NUM\cite{PalomarMungdecomposition:2006}, the iterative updates (with
message exchange)\footnote{Since the iterations within a CSI
coherence time involve explicit message passing, there is processing
and signaling overhead per iteration and this limits the total
number of iterations within a CSI coherence time.} are performed
within the CSI coherence time and hence, this limits the number of
iterations and the performance. However, in the proposed online
algorithm,  the updates evolves in the same time scale as the CSI
and QSI. Hence, it could converge to a better solution because the
number of iterations is no longer limited by the coherence time of
CSI.~ \hfill\QED
\end{Rem}

\begin{Rem}(\emph{Comparison to the Conventional Reinforced Learning})
There are two key novelties in the proposed per-user online update
algorithms. Firstly, most of the existing literature regarding
online learning addressed unconstrained MDP only\cite{Cao:2007}. In
the case of CMDP, the LM are determined offline by
simulation\cite{Djoin:2007}. In our case, both the LM and the
per-user Q-factor are updated simultaneously. Secondly, conventional
online learning are designed for centralized solution where the
control actions are determined entirely from the potential or
Q-factor update. However, in our case, the control actions for user
$k$ are determined from $\{\mathcal Q^k(\varphi)\}$ ($\forall k$)
via a per-stage auction. During the iterative updates, both the
per-user Q-factor/LMs as well as the control actions are changed
dynamically and the existing convergence results (based on
contraction mapping argument) cannot be applied directly to our
distributive stochastic learning algorithm.~ \hfill\QED
\end{Rem}

\subsection{Convergence Analysis}\label{subsec_convergence distri_alg}




In this section, we shall establish technical conditions for the
almost-sure convergence of the online distributive learning
algorithm.    For any LM $\boldsymbol{\gamma}$ $(\gamma^k \geq 0)$,
define a vector mapping $\mathbf{T}^k: R^2\times
R^{I_{\varphi}}\rightarrow R^{I_{\varphi}}$ for user $k$, and
$\mathbf{T}^k\triangleq(T^k_1,\cdots,T^k_{I_{\varphi}})^T$ with the
$i$-th ($1\leq i \leq I_{\varphi}$) component mapping defined as
$T_i^k(\gamma^k,\boldsymbol{\mathcal{Q}}^k)\triangleq\min_{\mathbf{p}_k}
\Big[ g_k(\boldsymbol{\gamma}^k,\varphi^i,\mathbf{p}_k) +
\sum_{\varphi^j} \Pr[\varphi^j|\varphi^i,\mathbf{p}_k]
\mathcal{Q}^k(\varphi^j)\Big]$, where $
\Pr[\varphi^j|\varphi^i,\mathbf{p}_k]=\Pr[\boldsymbol{\chi}^j_k,
\mathbf{s}^j_k|\varphi^i,\mathbf{p}_k]=\Pr[\boldsymbol{\chi}^j_k|\varphi^i,\mathbf{p}_k]\Pr[\mathbf{s}^j_k|\boldsymbol{\chi}^j_k]
=\Pr[\boldsymbol{\chi}^j_k|\varphi^i,\mathbf{p}_k]\prod_n
\Pr[s^j_{k,n}(|H^j_{k,n}|\geq H^*_{K-1})+(1-s^j_{k,n})(|H^j_{k,n}|<
H^*_{K-1})|H^j_{k,n}]$. Define
\begin{align}
&\mathbf{A}_{t-1}^k\triangleq
P^k_t\epsilon_{t-1}^v+(1-\epsilon_{t-1}^v) \mathbf{I},\nonumber\\
&\mathbf{B}_{t-1}^k\triangleq
P^k_{t-1}\epsilon_{t-1}^v+(1-\epsilon_{t-1}^v) \mathbf{I}
\end{align}
where $P^k_t$ is the $I_{\varphi}\times I_{\varphi}$ transition
probability matrix with $\Pr[\varphi^j|\varphi^i,\mathbf{p}^k_t(i)]$
as its $(i,j)$-element, where $\mathbf{p}^k_t(i)$ denotes the power
allocation for $\varphi^i$ obtained by per-stage subband auction at
the $t$-th iteration, and $\mathbf{I}$ is the $I_{\varphi}\times
I_{\varphi}$ identity matrix.

Since we have two different step size sequences
$\{\epsilon_t^{\gamma}\}$ and $\{\epsilon_t^{q}\}$ and
$\epsilon_t^{\gamma} = \boldsymbol{o}(\epsilon_t^{q})$, the LM
updates and the per-user Q-factor updates are done simultaneously
but over two different time scales. During the per-user  Q-factor
update (timescale I), we have
$\overline{\gamma}_{t+1}^k-\overline{\gamma}^k_t=e(t)$ and
$\underline{\gamma}_{t+1}^k-\underline{\gamma}^k_t=e(t)$ $(\forall k
)$, where
$e(t)=\mathcal{O}(\epsilon_t^{\gamma})=\boldsymbol{o}(\epsilon_t^{q})$.
Therefore, the LM appears to be quasi-static \cite{Borkarbook:2008}
during the per-user Q-factor update in \eqref{eqn:Q-est1}. We first
have the following lemma.

\begin{Lem}(\emph{Convergence of Per-user Q-factor Learning over  Timescale
I}) Assume for all the feasible policies $\Omega$ in the policy
space, there exists a $\delta_m=\mathcal{O}(\epsilon_m^q)>0$ and
some positive integer $m$ such that
\begin{eqnarray}
[\mathbf{A}_{m}^k\cdots\mathbf{A}_{1}^k]_{i r}\geq \delta_m, \
[\mathbf{B}_{m}^k\cdots\mathbf{B}_{1}^k]_{i r}\geq \delta_m, \quad
1\leq i \leq I_{\varphi}\label{promatrix-cond-dis}
\end{eqnarray}
where $[\cdot]_{i r}$ denotes the element of the $i$-th row with
$r$-th column of the corresponding $I_{\varphi}\times I_{\varphi}$
matrix ($r$ is the column index in $P_t^k$ which contains the
aggregate reference state $\varphi^r$). For stepsize sequence
$\{\epsilon_t^{q}\},\{\epsilon_t^{\gamma}\}$ satisfying the
conditions in \eqref{stepsize}, we have $\lim_{t\rightarrow
\infty}\boldsymbol{\mathcal{Q}}^k_t=\boldsymbol{\mathcal{Q}}^k_{\infty}(\boldsymbol{\gamma})$
$\forall k$ a.s. for any initial per-user subband allocation
Q-factor vector $\boldsymbol{\mathcal{Q}}^k_0$ and LM
$\boldsymbol{\gamma}$, where
the converged per-user subband allocation Q-factor
$\boldsymbol{\mathcal{Q}}^k_{\infty}(\boldsymbol{\gamma})$
satisfies:
\begin{align}
\big(T_r^k(\gamma^k,\boldsymbol{\mathcal{Q}}^k_{\infty}(\boldsymbol{\gamma}))-\mathcal{Q}^k_{\infty}(\varphi^r)
\big)\mathbf{e}+\boldsymbol{\mathcal{Q}}^k_{\infty}(\boldsymbol{\gamma})
=\mathbf{T}^k(\gamma^k,\boldsymbol{\mathcal{Q}}^k_{\infty}(\boldsymbol{\gamma}))
\label{bellmaneqn-wp1-dis-fixed-gamma}
\end{align}\label{Lem:convergence-dis-fixed-gamma}
\end{Lem}
\begin{proof}
Please refer to Appendix B.
\end{proof}

%

On the other hand, during the LM update (timescale II), we have
$\lim_{t\rightarrow
\infty}\|\boldsymbol{\mathcal{Q}}^k_t-\boldsymbol{\mathcal{Q}}^k_{\infty}(\boldsymbol{\gamma}_t)\|=0$
w.p.1. by the Corollary 2.1 of \cite{Borkartwotimescales:1997}.
Hence, during the LM updates in \eqref{eqn:gamma-dis-est1} and
\eqref{eqn:gamma-dis-est1-under}, the per-user subband allocation
Q-factor update is seen as almost equilibrated. The convergence of
the LM is summarized below.

\begin{Lem}[Convergence of the LM over Timescale II]
The iterates $\lim_{t\to \infty}
\boldsymbol{\gamma}_t=\boldsymbol{\gamma}_{\infty}$ a.s., where
$\boldsymbol{\gamma}_{\infty}$ satisfies the power and packet drop
rate constraints in \eqref{eqn:tx-pwr2} and \eqref{eqn:ptdrop2}. ~
\hfill\QED
 \label{Lem:convergence-dis-gamma}
\end{Lem}
\begin{proof}
Please refer to Appendix C.
\end{proof}

Based on the above lemmas, we shall summarize the convergence
performance of the online per-user Q-factor and LM learning
algorithm in the following theorem.

\begin{Thm}
[Convergence of Online Per-user Learning Algorithm] For the same
conditions as in Lemma \ref{Lem:convergence-dis-fixed-gamma},  we
have $(\boldsymbol{\mathcal{Q}}^k_t, \gamma_t^k)\longrightarrow
(\boldsymbol{\mathcal{Q}}^k_{\infty}, \gamma^k_{\infty})$ a.s.
$\forall k $, where
$\boldsymbol{\mathcal{Q}}^k_{\infty}(\boldsymbol{\gamma}_{\infty})$
and $\boldsymbol{\gamma}_{\infty}$ satisfy
\begin{eqnarray}
\big(T_r^k(\gamma^k_{\infty},\boldsymbol{\mathcal{Q}}^k_{\infty})-\mathcal{Q}^k_{\infty}(\varphi^r)
\big)\mathbf{e}+\boldsymbol{\mathcal{Q}}^k_{\infty}
=\mathbf{T}^k(\gamma^k_{\infty},\boldsymbol{\mathcal{Q}}^k_{\infty})
\label{bellmaneqn-wp1-dis-optimal-gamma}
\end{eqnarray}
and $\boldsymbol{\gamma}_{\infty}$ satisfies the power and packet
drop rate constraints in \eqref{eqn:tx-pwr2} and
\eqref{eqn:ptdrop2}. ~ \hfill\QED
 \label{thm:convergence-dis}
\end{Thm}

%


\section{Application to the OFDMA systems with Exponential Packet Size Distribution} \label{sec_application}
In this section, we shall illustrate the application of the proposed
stochastic learning algorithm by an example with exponential packet
size distribution.

\subsection{Dynamics of System State under Exponential Distributed
Packet Size}

Let $\mathbf{A}(t)=\big(A_1(t),\cdots,A_K(t)\big)$ and
$\mathbf{N}(t)=\big(N_1(t),\cdots,N_K(t)\big)$ be be the random new
packet arrivals and the packet sizes for the K users at the $t$-th
scheduling slot, respectively.
$\mathbf{Q}(t)=\big(Q_1(t),\cdots,Q_K(t)\big)$ and $N_Q$ denote the
joint QSI (number of packets) at the end of the $t$-th scheduling
slot and the maximum buffer size (number of packets).

\begin{Asump}
The arrival process $A_k(t)$ is i.i.d. over scheduling slots
according to a general distribution $\Pr(A_k)$ with average arrival
rate $\mathbb{E}[A_k]=\lambda_k$. The random packet size $N_k(t)$ is
i.i.d. over scheduling slots  following an exponential distribution
with mean packet size $\overline{N}_k$.~ \hfill\QED
\label{Asump:A-N}
\end{Asump}

Given a stationary policy, define the conditional mean departure
rate of packets of user $k$ at the $t$-th slot (conditioned on
$\boldsymbol{\chi}(t)$) as
$\mu_k(\boldsymbol{\chi}(t))=R_k(\boldsymbol{\chi}(t))/
\overline{N}_k$.
\begin{Asump}
The slot duration $\tau$ is sufficiently small compared with the
average packet service time, i.e. $\mu_k(\boldsymbol{\chi}(t)) \tau
\ll 1$\footnote{This assumption is reasonable in practical systems.
For instance, in the UL WiMAX (with multiple UL users served
simultaneously), the minimum resource block that could be allocated
to a user in the UL is $8 \times 16$ symbols $-$ 12 pilot
symbols$=$116 symbols. Even with 64QAM and rate $\frac{1}{2}$
coding, the number of payload bits it can carry is
$116\times3$bits$=$348 bits. As a result, when there are a lot of UL
users sharing the WiMAX AP, there could be cases that the MPEG4
packet (around 10K bits) from an UL user cannot be delivered in one
frame. In addition, the delay requirement of MPEG4 is 500ms or more,
while the frame duration of Wimax is 5ms. Hence, it is not necessary
to serve one packet during one scheduling slot so that the scheduler
has more flexibility in allocating resource. Therefore, in practical
systems, an application level packet may have mean packet length
spanning over many time slots (frames) and this assumption is also
adopted in
\cite{Sadiq:2009,Baris:2009,Jennifer:2009,Crabill:1972}.}.~
\hfill\QED \label{Asump:tau}
\end{Asump}
Given the current system state $\boldsymbol{\chi}(t)$ and the
control action, and conditioned on the packet arrival
$\mathbf{A}(t)$ at the end of the $t$-th slot, there will be a
packet departure of the $k$-th user at the $(t+1)$-th slot if the
remaining service time of a packet is less than the current slot
duration $\tau$. By the memoryless property of the exponential
distribution, the remaining packet length (also denoted as $\mathbf
N(t)$) at any slot $t$ is also exponential distributed. Hence, the
transition probability to $Q_k(t+1)$ at the $(t+1)$-th slot
corresponding to a packet departure event is given by\footnote{Since
$N_k(t)$ is exponentially distributed and is memoryless, we have the
probability in \eqref{tran-dep-1} (conditioned on the current state
$\boldsymbol{\chi}(t)$ and the associated action
$\Omega(\boldsymbol{\chi}(t))$ ) independent of the previous states
$\{\boldsymbol{\chi}(t-1),\boldsymbol{\chi}(t-2),\cdots \}$.}:
\begin{align}
&\Pr[Q_k(t+1) = A_k(t) + Q_k(t) - 1 |
\boldsymbol{\chi}(t),\mathbf{A}(t),\Omega(\boldsymbol{\chi}(t))]
\nonumber\\
=& \Pr \Big[ \frac{N_k(t)}{R_k(t)} < \tau |
\boldsymbol{\chi}(t),\mathbf{A}(t),\Omega(\boldsymbol{\chi}(t))
\Big] \nonumber
\\
=& \Pr \Big[\frac{N_k(t)}{\overline{N_k}} <
\mu_k(\boldsymbol{\chi}(t))\tau \Big]\nonumber\\
=& 1 - \exp (-\mu_k(\boldsymbol{\chi}(t))\tau) \approx
\mu_k(\boldsymbol{\chi}(t)) \tau \label{tran-dep-1}
\end{align}
where the last equality is due to Assumption \ref{Asump:tau}.  Note
that the probability for simultaneous departure of two or more
packets from the same queue or different queues in a slot is
$\mathcal{O}((\mu_k(\boldsymbol{\chi}(t)) \tau)^2)$, which is
asymptotically negligible. Therefore, the vector queue dynamics is
Markovian with the transition probability given by
\begin{align}
&\Pr[\mathbf{Q}(t+1)|
\boldsymbol{\chi}(t),\Omega(\boldsymbol{\chi}(t))]\nonumber\\
=&\sum_k
\Pr[\mathbf{A}(t)=\mathbf{Q}(t+1)-\mathbf{Q}(t)+\mathbf{e}_k]
\mu_k(\boldsymbol{\chi}(t))
\tau\nonumber\\
&+\Pr[\mathbf{A}(t)=\mathbf{Q}(t+1)-\mathbf{Q}(t)]\big(1-\sum_k\mu_k(\boldsymbol{\chi}(t))
\tau\big) \label{vec-q-trans-prob}
\end{align}
where $\mathbf{e}_k$ denotes the standard basis vector with 1 for
its $k$-th component and 0 for every other component.

\subsection{Decomposition of the Per-user Subband Allocation Q-factor}\label{subsec_decomposition}

In the following lemma, we shall show that the per-user subband
allocation Q-factor
$\mathcal{Q}^k(\boldsymbol{\chi}_k,\mathbf{s}_k)$ can be further
decomposed into the sum of per-user per-subband Q-factor, which
further simplifies the learning algorithm.


\begin{Lem}[Decomposition of Per-user Q-factor]
The {\em per-user  Q-factor}
$\mathcal{Q}^k(\boldsymbol{\chi}_k,\mathbf{s}_k)$  (defined by the
fixed point equation in \eqref{eqn:Bellman-per-user-q-factor}) can
be decomposed into the sum of the {\em per-user per-subband
Q-factor} $\{q^k(Q,|H|,s)\}$, i.e.
$\mathcal{Q}^k(\boldsymbol{\chi}_k,\mathbf{s}_k)=\sum_n
q^k(Q_k,|H_{k,n}|,s_{k,n})$, where
\begin{align}
&q^k(Q_k,|H_{k,n}|,s_{k,n}) \triangleq
\min_{p_{k,n}}\big\{g_{k,n}(\boldsymbol{\gamma}^k,Q_k,|H_{k,n}|,s_{k,n},p_{k,n})\nonumber\\
& \quad \quad \quad \quad \quad \quad \quad-\frac{N_F\delta
\widetilde{w}^k(Q_k)\tau}{\overline{N_k}}s_{k,n}\log(1 +
p_{k,n}|H_{k,n}|^2)
\nonumber\\
&\quad \quad \quad \quad \quad\quad \quad+\mathbb{E}[\widetilde{w}^k(Q_k+A_k)|Q_k]-\frac{\theta^k}{N_F}\big\}\label{eqn:per-user-per-subb-q-factor}\\
&g_{k,n}(\boldsymbol{\gamma}^k,Q_k,|H_{k,n}|,s_{k,n},p_{k,n})=\overline{\gamma}^k
p_{k,n}+\frac{1}{N_F}\big(\beta_k
f(Q_k)\nonumber\\
&\quad \quad \quad \quad \quad \quad \quad -\overline{\gamma}^k
P_k+\underline{\gamma}^k(\mathbf{1}[Q_k=N_Q]-P_k^d)\big)\\
&\widetilde{w}^k(Q_k)=\mathbb{E}\Big[ q^k\big(Q_k,
|H_{k,n}|,s_{k,n}=\mathbf{1}[|H_{k,n}|\geq H^*_{K-1}]\big)|Q_k
\Big]\\
&\delta\widetilde{w}^k(Q_k)=\mathbb{E}[\widetilde{w}^k(Q_k+A_k)-\widetilde{w}^k(Q_k+A_k-1)|Q_k]
\end{align}
Furthermore, we have $\widetilde{W}^k(Q_k)=N_F
\widetilde{w}^k(Q_k)$.~ \hfill\QED
\label{Lem:reduced-MDP-opt-control-per-user-per-subband-q-factor}
\end{Lem}
\begin{proof}
Please refer to Appendix D for the proof.
\end{proof}

Based on the {\em per-user per-subband Q-factor} $\{q^k(Q,|H|,s)\}$,
we can obtain the closed-form power allocation actions minimizing
the R.H.S. of the {\em per-user subband allocation Q-factor fixed
point equation} in \eqref{eqn:Bellman-per-user-q-factor}, which is
summarized in the following lemma:

\begin{Lem} [Decentralized Power Control Actions] Given subband allocation actions $\mathbf{s}_k$, the
optimal power control actions of user $k$ under the linear
approximation on subband allocation Q-factor in
\eqref{eqn:approximate-Q} are given by
\begin{align}
p_{k,n}(Q_k, H_{k,n})=s_{k,n}\Big(\frac{\frac{\tau}{\overline{N}_k}
N_F\delta\widetilde{w}^k(Q_k)}{\overline{\gamma}^k}-\frac{1}{|H_{k,n}|^2}\Big)^{+},
\ \forall n \label{eqn:joint-pwr-allo}
\end{align}
\label{Lem:per-user-power-allocation}
\end{Lem}
\begin{proof}
Please refer to Appendix E for the proof.
\end{proof}
\begin{Rem}(\emph{Multi-level Water-filling Structure of the  Power Control Action}) The  power control
action in \eqref{eqn:joint-pwr-allo} of Lemma
\ref{Lem:per-user-power-allocation} is both function of CSI and QSI
(where it depends on the QSI indirectly via $\delta
\widetilde{w}^k(Q_k)$, which is function of $\{q^k(Q,|H|,s)\}$). It
has the form of {\em multi-level water-filling} where the power is
allocated according to the CSI across subbands but the water-level
is adaptive to the QSI. ~ \hfill\QED
\end{Rem}

\subsection{Per-Stage Per-Subband Auction}\label{subsec_auction}

\begin{figure}[t]
\begin{center}
\includegraphics[height=6cm, width=9cm]{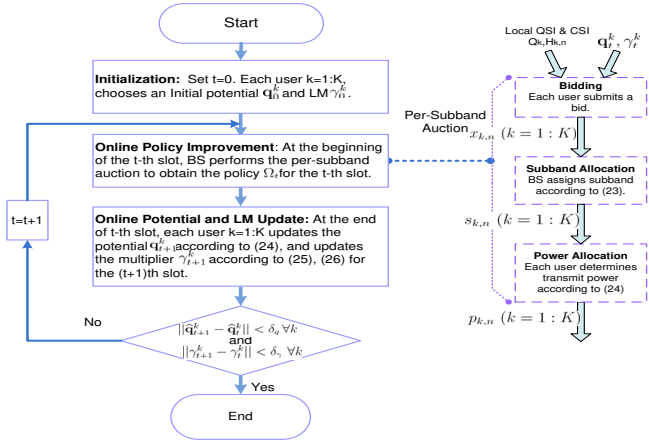}
\caption{Algorithm Flow of the Online Distributive Primal-Dual Value
Iteration Algorithm with Per-stage Auction and Simultaneous Updates
on Potential and Lagrange multipliers (LM). Note that
$t=\{0,1,2,...\}$ is the scheduling slot index.} \label{algorithm2}
\end{center}
\end{figure}

Applying the per-stage subband auction in Section
\ref{subsec_auction} to the system dynamics setup in this section,
we obtain a low computational complexity and signaling overhead {\em
Scalarized Per-Subband Auction} ($\forall n\in\{1,N_F\}$) as
illustrated in Fig. \ref{algorithm2}, which is based on the per-user
subband allocation Q-factor decomposition in Lemma
\ref{Lem:reduced-MDP-opt-control-per-user-per-subband-q-factor} and
the closed-form power allocation actions in Lemma
\ref{Lem:per-user-power-allocation} as follows:
\begin{itemize}
\item \textbf{Bidding}: For the $n$-th subband, each user submits a bid
\begin{align}
X_{k,n}=&\frac{N_F\delta \widetilde{w}^k(Q_k)\tau}{\overline{N_k}}
\log\Big(1+|H_{k,n}|^2\big(\frac{\frac{N_F\delta
\widetilde{w}^k(Q_k)\tau}{\overline{N_k}}}{\overline{\gamma}^k}\nonumber\\
&-\frac{1}{|H_{k,n}|^2}\big)^{+}\Big)-\overline{\gamma}^k\big(\frac{\frac{N_F\delta
\widetilde{w}^k(Q_k)\tau}{\overline{N_k}}}{\overline{\gamma}^k}-\frac{1}{|H_{k,n}|^2}\big)^{+}\nonumber
\end{align}
\item \textbf{Subband Allocation}: The BS assigns the $n$-th subband
according to the highest bid:
\begin{eqnarray}
 s^*_{k,n}(\mathbf{H}_n, \mathbf{Q}) = \left\{
\begin{array}{ll} 1, &
\text{if}\quad k=k^*_n \quad \text{and} \quad X_{k^*_n,n} >0\\
0, & \textrm{otherwise}
\end{array} \right.\label{eqn:per-subc-subc-allo}
\end{eqnarray}
where $k^*_n=\arg\max_k X_{k,n}$ denotes the user with the highest
bid and then broadcasts the allocation results to $K$ users.
\item \textbf{Power Allocation}: Each user determines the transmit power
according to:
\begin{align}
p^*_{k,n}(\mathbf{H}_n, \mathbf{Q})=&s^*_{k,n}(\mathbf{H}_n,
\mathbf{Q})\Big(\frac{\frac{\tau}{\overline{N}_k}
N_F\delta\widetilde{w}^k(Q_k)}{\overline{\gamma}^k}\nonumber\\
&-\frac{1}{|H_{k,n}|^2}\Big)^{+}\label{eqn:per-subc-pwr-allo}
\end{align}
\end{itemize}

\begin{Rem}[Comparison to Brute-Force (CSI,QSI)-Feedback] In the brute-force (CSI,QSI)-feedback
scheme, each MS $k$ needs to feedback CSI $|H_{k,n}|$ ($\forall n$),
QSI $Q_k$ and the LMs $\gamma_k$. BS needs to solve the subband
allocation $s^*_{k,n}$ and power allocation $p^*_{k,n}$, and
broadcast
the (real number) power allocation $p^*_{k,n}$ to the MSs. 
Note that for the signaling from MS to BS, the quantization bits
used in signaling for the bid $X_{k,n}$ versus for the CSI
$|H_{k,n}|$ is similar. However, the proposed per-subband auction
does not need to feedback QSI and LM. For the signaling from BS to
MS, the proposed per-stage auction only needs 1 bit per subband for
$s^*_{k,n}$. However, the brute-force (CSI,QSI)-feedback scheme
needs much more bits per subband for a relatively accurate
$p^*_{k,n}$ to ensure acceptable performance. Therefore, compared
with the brute-force (CSI,QSI)-feedback scheme for uplink OFDMA
systems, the proposed scalarized per-subband auction greatly reduces
the signaling overhead and computation complexity (at the BS) for
subband allocation and power allocation in the decentralized
solution.~ \hfill\QED
\end{Rem}


\subsection{Online Per-user Primal-Dual  Learning Algorithm via Stochastic Approximation}\label{subsec_online distri_alg}

In this part, we shall apply the online localized primal-dual
learning algorithm in Section \ref{subsec_online distri_alg} to
estimate $\{q^k(Q,|H|,s)\}$ and LMs. The update equations for LMs
are the same as \eqref{eqn:gamma-dis-est1} and
\eqref{eqn:gamma-dis-est1-under}, and hence, we shall focus on the
online learning of per-user per-subband Q-factor $\{q^k(Q,|H|,s)\}$
in the following.

 For
notation convenience, we denote the per-user per-subband
state-action pair as $\phi\triangleq(Q,|H|,s)$. Let $i$ ($1\leq i
\leq I_{\phi} $) be a dummy index enumerating over all the possible
state-action pairs of each user over one subband with cardinality
$I_{\phi}=2 N_H(N_Q+1)$ and $\phi_{k,n} (t)\triangleq\big(Q_k(t),
|H_{k,n}(t)|, s_{k,n}(t)\big)$ be the current state-action pair
observed at MS $k$ on subband $n$ at the $t$-th slot.  Based on the
current observation $\phi_{k,n} (t)$, user $k$ updates its estimate
on the per-user per-subband Q-factor according to:
\begin{align}
\ q^k_{t+1}(\phi^i)=&
       q^k_t(\phi^i) + \epsilon_{l_k(\phi^i,t)}^{q}
\big[g_{k,n^k_i}(\boldsymbol{\gamma}^k_t,\phi^i,p_{k,n^k_i}(t))\nonumber \\
&+ \widetilde{w}^k_t(Q_k(t+1))\big)
-\big(g_{k,\bar{n}^k_I}(\boldsymbol{\gamma}^k_t,\phi^I,p_{k,\bar{n}^k_I}(\bar{t}))\nonumber\\
&+ \widetilde{w}^k_t(Q_k(\bar t +1)) -q^k_t(\phi^I)\big) -
 q^k_t(\phi^i) \big]\nonumber\\
& \cdot \mathbf{1}\big[\cup_n\{\phi_{k,n} (t)=\phi^i\} \big]
\label{eqn:q-est1}
\end{align}
where $l_k(\phi^i,t)\triangleq\sum_{m=0}^t
\mathbf{1}\big[\cup_n\{\phi_{k,n} (m)=\phi^i\} \big]$ is the number
of updates of $q^k(\phi^i)$ till $t$\cite{Borkarasynchronous:1998},
$n^k_i\in \{n: \phi_{k,n} (t)=\phi^i \}$\footnote{$\forall n^k_i\in
\{n: \phi_{k,n} (t)=\phi^i \}$,
$g_{k,n^k_i}(\boldsymbol{\gamma}^k_t,\phi^i,p_{k,n^k_i}(t))$ is
equal.}, $\bar{t}\triangleq \sup\{t:\phi_{k,n}(t)=\phi^I\}$,
$\phi^I$ is the reference (per-subband) state-action
combination\footnote{The reference (per-user) state-action
combination $\varphi^r$ is composed of the (per-subband)
state-action combination $\phi^I$. For example, say $N_F = 2$, $Q =
\{0,1\}$, $|H| = \{\text{Good (G), Bad (B)}\}$, $s = \{0, 1\}$,
$I_{\varphi} = 2\times2^2\times 2^2 = 48$, $I_{\phi} =
2\times2\times 2 = 8$. Let $\phi^I=(0,\text{B},0)$, then
$\varphi^r=(0,\{\text{B,B}\}, \{0,0\})$ (aggregated over 2
subbands). Without loss of generality, we initialize  the per-user
per-subband Q-factor as 0, i.e. $q^k_0(\phi^I)=0\ \forall k$.}
(per-user per-subband), $\bar{n}^k_I\in \{n: \phi_{k,n}
(\bar{t})=\phi^I \}$. 

\subsection{Rate of Convergence and Asymptotic Performance}
In this section, we shall discuss the convergence speed as well as
the asymptotic performance of the proposed distributive stochastic
learning algorithm. For instance, we are interested in how the
convergence speed scales with the number of MS $K$ and the number of
subbands $N$. In the asynchronous per-user per-subband Q-factor
learning algorithm, at slot $t$, each user $k$ updates the Q-factor
of all the per-user per-subband state-action pairs observed in $N$
subbands. Thus, the convergence speed of the asynchronous per-user
per-subband Q-factor learning algorithm depends on the speed that
every per-user per-subband state-action pair of each user $k$ is
visited at the steady state. We define the {\em ergodic visiting
speed} for each MS $k$ as $V_k=\lim_{t\to \infty} \frac{\min_i
l_k(\phi^i,t)}{t}$, where $l_k(\phi^i,t)\triangleq\sum_{m=0}^t
\mathbf{1}\big[\cup_n\{\phi_{k,n} (m)=\phi^i\} \big]$ is the number
of updates of $q^k(\phi^i)$ up to slot $t$. The following lemma
summarizes the main results regarding the {\em ergodic visiting
speed}.

\begin{Lem}[Ergodic Visiting Speed w.r.t. K and
N] The {\em ergodic visiting speed} for each MS $k$ of the per-user
per-subband Q-factor stochastic learning algorithm in
\eqref{eqn:q-est1} is given by $V_k=\mathcal O(N/K)$ ($\forall k$).
~ \hfill\QED \label{Lem:order-of-growth}
\end{Lem}
\begin{proof}
Please refer to Appendix F.
\end{proof}

\begin{Rem}[Interpretations] Note that the convergence rate of the
learning algorithm is related to  $ V_k=\mathcal O(N/K)$. Observe
that the convergence speed increases as $N$ increase. This is
because in the asynchronous update process in \eqref{eqn:q-est1},
each user $k$ updates the Q-factor of all the per-user per-subband
state-action pair observed in $N$ subbands in a single time slot.
Hence,  there is intrinsic {\em parallelism} in the learning process
across different subbands.
\end{Rem}

Finally, we shall show that the performance of the distributive
algorithm is asymptotically global optimal for large number of
users.

\begin{Thm} [Asymptotically Global Optimal]
For sufficiently large $K$ such that the optimization Problem
\ref{Prob2} is feasible, the performance of the online distributive
per-user primal-dual  learning algorithm is asymptotically global
optimal, i.e. $\sum_{k=1}^K
\mathcal{Q}^k_{\infty}(\boldsymbol{\chi}_k,\mathbf{s}_k)\rightarrow
\mathcal{Q}^*(\boldsymbol{\chi},\mathbf{s})$ and
$\boldsymbol{\gamma}_{\infty} \rightarrow \boldsymbol{\gamma}^*$ as
$K\rightarrow \infty$, where
$\mathcal{Q}^*(\boldsymbol{\chi},\mathbf{s})$ and
$\boldsymbol{\gamma}^*$ are the solution of the centralized Bellman
equation in \eqref{eqn:Bellman-q-factor}  satisfying the
corresponding constraints in (\ref{eqn:tx-pwr2}),
\eqref{eqn:ptdrop2}. ~ \hfill\QED
\label{Thm:asymp-performance-dis-alg}
\end{Thm}
\begin{proof}
Please refer to Appendix G.
\end{proof}

\section{Simulation Results and Discussions}\label{sec_numerical_results}

In this section, we shall compare our proposed per-user online
learning algorithm via stochastic approximation to the delay optimal
problem for OFDMA uplink systems with the centralized  subband
allocation Q-factor $\{\mathcal{Q}(\boldsymbol{\chi},\mathbf{s})\}$
learning algorithm and three other reference baselines. Baseline 1
refers to a throughput optimal policy\footnote{Throughput optimal
policy means that it shall stabilize the queue whenever the arrival
rate vector falls within the stability region.}, namely the {\em
Modified Largest Weighted Delay First (M-LWDF)\cite{MagMLWDF:2001},
in which the subband and power control are chosen to maximize the
weighted delay.} Baseline 2 refers to the {\em CSIT Only
Scheduling}, in which optimal subband and power allocation is
performed purely based on CSIT. Baseline 3 refers to the {\em Round
Robin Scheduling}, in which different users are served in TDMA
fashion with equally allocated time slots and water-filling power
allocation across the subbands. In the simulation, we consider
Poisson packet arrival with average arrival rate $\lambda_k$ (pck/s)
and exponential packet size distritution with mean $\overline N_k$.
We consider average delay as our utility ($f(Q_k) =
\frac{Q_k}{\lambda_k}$). We assume there are 64 subbands with total
BW 10MHz, and the number of independent subbands $N_F$ is 4. The
scheduling slot duration $\tau$ is 5ms. The buffer size $N_Q$ is 10.

\begin{figure}[t]
\begin{center}
\includegraphics[height=5.5cm, width=8cm]{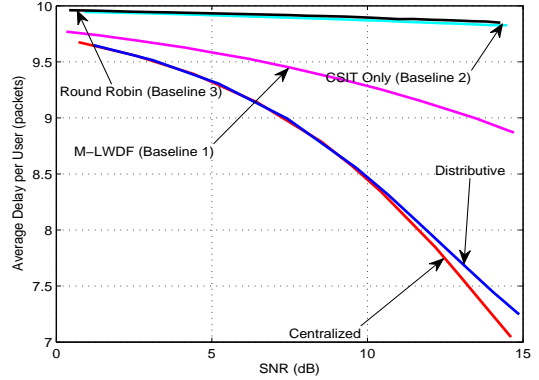}
\caption{Average delay per user versus SNR. The number of users
$K=2$, the buffer size $N_Q=10$, the mean packet size
$\overline{N}_k=305.2$ Kbyte/pck, the average arrival rate
$\lambda_k=20$ pck/s, the queue weight $\beta_{1}=\beta_{2}=1$. The
packet drop rate of the proposed scheme is 5$\%$  while the packet
drop rate of the Baseline 1 (M-LWDF), Baseline 2 (CSIT Only) and
Baseline 3 (Round Robin) are 5$\%$, 8$\%$, 9$\%$ respectively. }
\label{10N_Q_delay_vs_SNR}
\end{center}
\end{figure}

\begin{figure}[t]
\begin{center}
\includegraphics[height=5.5cm,width=8cm]{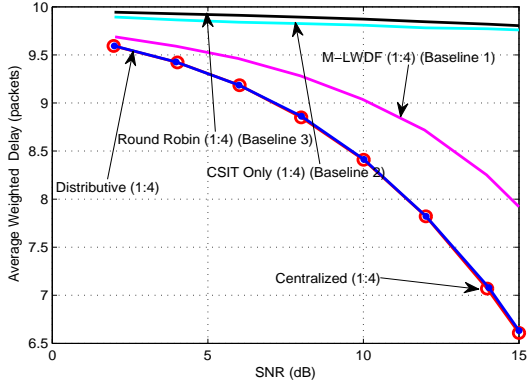} \caption{Average
weighted delay versus SNR. The number of users $K=2$, the buffer
size $N_Q=10$, the mean packet size $\overline{N}_k=305.2$
Kbyte/pck, the average arrival rate $\lambda_k=20$ pck/s, the queue
weight $\beta_1=1$, $\beta_2=4$. The packet drop rate of the
proposed scheme is 7$\%$  while the packet drop rate of the Baseline
1 (M-LWDF), Baseline 2 (CSIT Only) and
Baseline 3 (Round Robin) are 7$\%$, 9$\%$, 9$\%$ respectively. 
} \label{2K_difbeta_1_4_delay_vs_SNR}
\end{center}
\end{figure}

Figure~\ref{10N_Q_delay_vs_SNR} illustrates the average delay per
user versus SNR of 2 users. It can be observed that both the
centralized solution and the distributive solution have significant
gain compared with the three baselines (e.g. more than 7.5 dB gain
over M-LWDF when average delay per queue is less than 9 packets). In
addition, the delay performance of the distributive solution, which
is asymptotically global optimal in large number of users, is very
close to the performance of the optimal solution even in $K=2$.
Similar observations could be made in
Figure~\ref{2K_difbeta_1_4_delay_vs_SNR} where we plot the average
weighted delay versus SNR of two heterogeneous users.

\begin{figure}[t]
\begin{center}
\includegraphics[height=5.5cm, width=8cm]{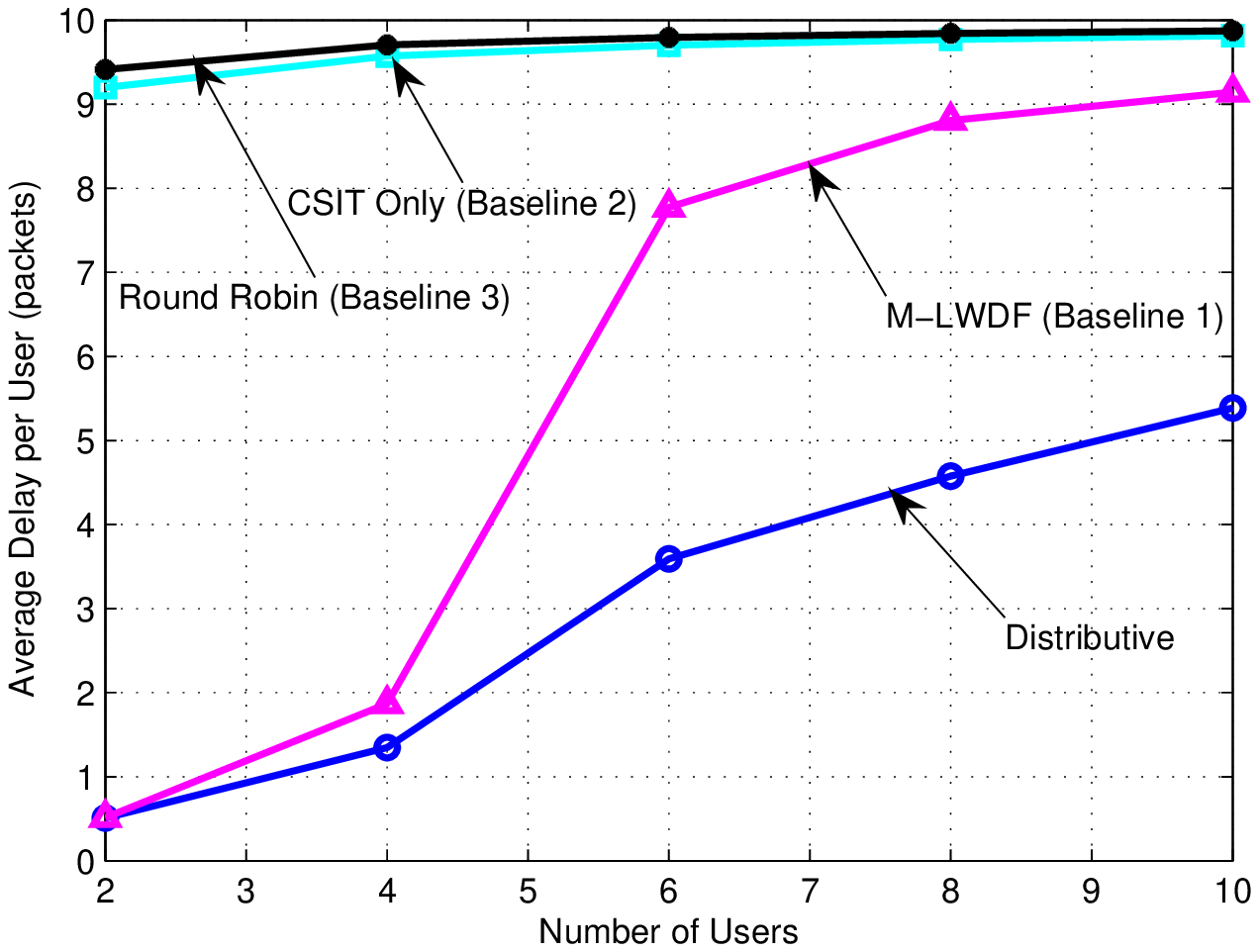}
\caption{Average delay per user versus the number of users. The
buffer size $N_Q=10$, the mean packet size $\overline{N}_k=78.125$
Kbyte/pck, the average arrival rate $\lambda_k=20$ pck/s, the queue
weight $\beta_k=1$ at a transmit SNR$=10$dB. The packet drop rate of
the proposed scheme is 4$\%$  while the packet drop rate of the
Baseline 1 (M-LWDF), Baseline 2 (CSIT Only) and Baseline 3 (Round
Robin) are 4$\%$, 8$\%$, 9$\%$ respectively.
}\label{delay_vs_largeK}
\end{center}
\end{figure}

Figure~\ref{delay_vs_largeK} illustrates the average delay per user
of the distributive solution versus the number of users at a
transmit SNR$=10$dB. It is obvious that the distributive solution
has significant gain in delay over the three baselines.
Figure~\ref{cdf_6K_10db} further illustrates the cumulative
distribution function (cdf) of the queue length for $K=6$ and
SNR$=10$dB. It can be seen that the distributive solution achieves a
smaller queue length compared with the other baselines.

Figure~\ref{convergence_10K} illustrates the convergence property of
the proposed algorithm. We plot the average
$\{\widetilde{W}^k(Q_k)\}$ of 10 users versus scheduling slot index
at a transmit SNR$=10$dB. It can be seen that the distributive
algorithm converges quite fast\footnote{In conventional iterative
algorithms for deterministic NUM, there is message passing between
iterative steps within a CSI realization and these iterative steps
(before convergence) are overheads because they do not carry useful
payload. On the other hand, the proposed algorithm is an online
distributive algorithm and hence, the slots before "convergence"
also carry useful payload and they are not ``wasted''.}. The average
delay corresponding to the average $\{\widetilde{W}^k(Q_k)\}$ at the
500-th scheduling slot is 5.9 pck, which is much smaller than the
other baselines.

\begin{figure}[t]
\begin{center}
\includegraphics[height=5.5cm, width=8cm]{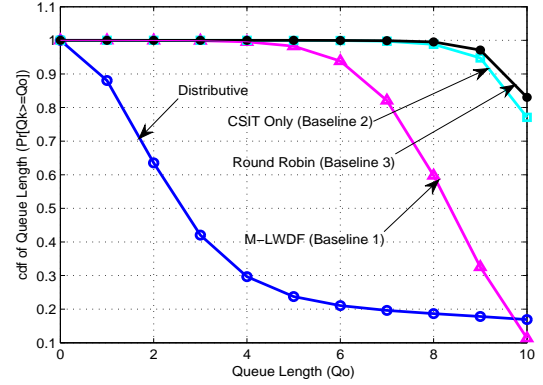}
\caption{Cumulative distribution function (cdf) of the queue length.
The buffer size $N_Q=10$, the mean packet size
$\overline{N}_k=78.125$ Kbyte/pck, the average arrival rate
$\lambda_k=20$ pck/s, the queue weight $\beta_k=1$, the number of
users $K=6$ at a transmit SNR$=10$dB. The packet drop rate of the
proposed scheme is 2$\%$  while the packet drop rate of the Baseline
1 (M-LWDF), Baseline 2 (CSIT Only) and Baseline 3 (Round Robin) are
2$\%$, 8$\%$, 8$\%$ respectively. }\label{cdf_6K_10db}
\end{center}
\end{figure}
\begin{figure}[t]
\begin{center}
\includegraphics[height=5.5cm, width=8cm]{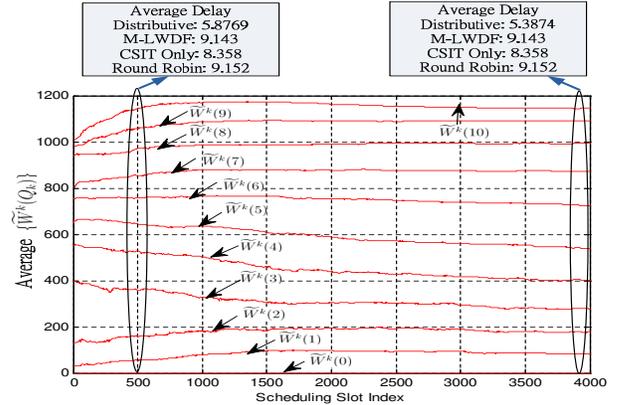}
\caption{Illustration of convergence property. The average
$\{\widetilde{W}^k(Q_k)\}$ of 10 users versus the scheduling slot
index. The number of users $K=10$, the buffer size $N_Q=10$, the
mean packet size $\overline{N}_k=78.125$ Kbyte/pck, the average
arrival rate $\lambda_k=20$ pck/s, the queue weight $\beta_k=1$ at a
transmit SNR$=10$dB. The packet drop rate of the proposed scheme is
4$\%$  while the packet drop rate of the Baseline 1 (M-LWDF),
Baseline 2 (CSIT Only) and Baseline 3 (Round Robin) are 4$\%$,
8$\%$, 9$\%$ respectively.}\label{convergence_10K}
\end{center}
\end{figure}

\section{Summary}\label{sec_summary}
In this paper, we consider a distributive delay-optimal power and
subband allocation design for uplink OFDMA system, which is cast
into an infinite-horizon average-reward CMDP.  To address the
distributive requirement and the issue of exponential memory
requirement and computational complexity, we proposed a per-user
online learning with per-stage auction, which requires local QSI and
local CSI only. We show that under the auction, the distributive
online learning converges with probability 1. For illustration, we
apply the proposed learning algorithm to an application example with
exponential packet size distribution. We show that the delay-optimal
power control has the {\em multi-level water-filling} structure.  We
show that the proposed algorithm converges to the global optimal
solution for sufficiently large number of users. Numerical results
illustrated significant delay performance gain over various
baselines.
\begin{appendix}

\section*{Appendix A: Proof of Lemma \ref{Lem:bellman-eqn-Q-factor}}
For a given $\boldsymbol{\gamma}$, the optimizing policy for the
unconstrained MDP in \eqref{uncons-MDP} can be obtained by solving
the {\em Bellman equation}  w.r.t.
$(\theta,\{V(\boldsymbol{\chi})\})$ as below\cite{Cao:2007}:
\begin{align}
&\theta + V(\boldsymbol{\chi}^i), \quad \quad \quad \quad \quad \quad \quad \quad  \quad \quad \forall 1\leq i\leq I_{\chi} \label{eqn:Bellman1} \\
=& \min_{\Omega(\boldsymbol{\chi}^i)} \Big[
g(\boldsymbol{\gamma},\boldsymbol{\chi}^i,
\Omega(\boldsymbol{\chi}^i)) + \sum_{\boldsymbol{\chi}^j}
\Pr[\boldsymbol{\chi}^j|\boldsymbol{\chi}^i,\Omega(\boldsymbol{\chi}^i)]
V(\boldsymbol{\chi}^j)\Big]\nonumber
\end{align}
where $\Omega(\boldsymbol{\chi}^i)=(\mathbf{p},\mathbf{s})$ is the
power control and subband allocation actions taken in state
$\boldsymbol{\chi}^i$, $\theta = L_{\beta}^*(\boldsymbol{\gamma}) =
\inf_{\Omega}L_{\beta}(\Omega,\boldsymbol{\gamma})$ is the optimal
average reward per stage, $\{V(\boldsymbol{\chi})\}$ is the
potential function of the MDP. Since
$\Omega(\boldsymbol{\chi}^i)=\big(\Omega_s(\boldsymbol{\chi}^i),\Omega_p(\boldsymbol{\chi}^i)\big)$,
we define the {\em subband allocation Q-factor} of state
$\boldsymbol{\chi}^i$ under subband allocation action $\mathbf{s}$
as $\mathcal{Q}(\boldsymbol{\chi}^i,\mathbf{s})\triangleq
\min_{\Omega_p(\boldsymbol{\chi}^i)} \Big[
g(\boldsymbol{\gamma},\boldsymbol{\chi}^i,\mathbf{s},\Omega_p(\boldsymbol{\chi}^i))
+ \sum_{\boldsymbol{\chi}^j}
\Pr[\boldsymbol{\chi}^j|\boldsymbol{\chi}^i,\mathbf{s},\Omega_p(\boldsymbol{\chi}^i)]
V(\boldsymbol{\chi}^j)\Big]-\theta $. Thus,
$V(\boldsymbol{\chi})=\min_{\mathbf{s}}\mathcal{Q}(\boldsymbol{\chi},\mathbf{s})$
($\forall \boldsymbol{\chi}$) and
$\{\mathcal{Q}(\boldsymbol{\chi},\mathbf{s})\}$ satisfy the Bellman
equation in \eqref{eqn:Bellman-q-factor}.

\section*{Appendix B: Proof of Lemma \ref{Lem:convergence-dis-fixed-gamma}}
Since $\forall k$, each state-action pair $\varphi^i$ is updated
comparably often \cite{Borkarasynchronous:1998}, the only difference
between the synchronous update and asynchronous update is that the
resultant ODE of the asynchronous update is a time-scaled version of
the synchronous update \cite{Borkarasynchronous:1998}. However, it
does not affect the convergence behavior. Therefore, we consider the
convergence of related synchronous version for simplicity in the
following.

Due to symmetry, we only consider the update for user $k$. 
It can be easily proved that the synchronous version of the per-user
Q-factor update in \eqref{eqn:Q-est1} is equivalent to the per-user
Q-factor update given by
\begin{align}
\mathcal{Q}^k_{t+1}(\varphi^i)=\mathcal{Q}^k_t(\varphi^i)+\epsilon^q_t
Y^k_t(\boldsymbol{\gamma}^k,\varphi^i)\quad 1\leq i \leq
I_{\varphi}\label{eqn:Q-factor-update}
\end{align}
where
$Y^k_t(\boldsymbol{\gamma}^k,\varphi^i)=g_k(\boldsymbol{\gamma}^k,\varphi^i,\mathbf{p}^k(t))
+\widetilde{W}^k_t(Q_k(t+1))
-\big(g_k(\boldsymbol{\gamma}^k,\varphi^r,\mathbf{p}^k(\bar{t}))+\widetilde{W}^k_t(\bar{Q}^r_k)-\mathcal{Q}^k_t(\varphi^r)\big)-\mathcal{Q}^k_t(\varphi^i)$.
Denote $\mathbf{Y}^k_t\triangleq
\big(Y^k_t(\boldsymbol{\gamma}^k,\varphi^1),\cdots,Y^k_t(\boldsymbol{\gamma}^k,\varphi^{I_{\varphi}})\big)^T$.
Let $\underline{\boldsymbol{\mathcal{Q}}}_t\triangleq
(\boldsymbol{\mathcal{Q}}^1_t,\cdots,\boldsymbol{\mathcal{Q}}^K_t)$
and $\underline{\mathbf{Y}}_t\triangleq
(\mathbf{Y}^1_t,\cdots,\mathbf{Y}^K_t)$ be the aggregate vector of
per-user Q-factor and $\mathbf{Y}^k_t$ (aggregate across all $K$
users in the system). We shall first establish the convergence of
the martingale noise in the Q-factor update dynamics.  Let
$\mathbb{E}_t$ and $\Pr_t$ denote the expectation and probability
conditioned on the $\sigma$-algebra $\mathcal{F}_t$, generated by
$\{\underline{\boldsymbol{\mathcal{Q}}}_0, \mathbf{\underline{Y}}_i,
i<t \}$, i.e. $\mathbb{E}_t[\cdot] =
\mathbb{E}[\cdot|\mathcal{F}_t]$ and $\Pr_t[\cdot]
=\Pr[\cdot|\mathcal{F}_t]$. Define
$R^k_t(\boldsymbol{\gamma}^k,\varphi^i)\triangleq\mathbb{E}_t[Y^k_t(\boldsymbol{\gamma}^k,\varphi^i)]=T_i^k(\gamma^k,\boldsymbol{\mathcal{Q}}^k_t)-\mathcal{Q}^k_t(\varphi^i)-\big(T_r^k(\gamma^k,\boldsymbol{\mathcal{Q}}^k_t)-\mathcal{Q}^k_t(\varphi^r)\big)$,
and $\delta M^k_t(\varphi^i)\triangleq
Y^k_t(\boldsymbol{\gamma}^k,\varphi^i)-\mathbb{E}_t[Y^k_t(\boldsymbol{\gamma}^k,\varphi^i)]$.
Thus, $\delta M^k_t(\varphi^i)$ is the martingale difference noise
satisfying the property that $\mathbb{E}_t[\delta
M^k_t(\varphi^i)]=0$ and $\mathbb{E}[\delta M^k_t(\varphi^i)\delta
M^k_{t{'}}(\varphi^i)]=0$ ($\forall t\neq t{'}$). For some $j$,
define $M^k_t(\varphi^i)=\sum_{l=j}^t \epsilon^q_l \delta
M^k_l(\varphi^i)$. Then, from \eqref{eqn:Q-factor-update}, we have
\begin{align}
\mathcal{Q}^k_{t+1}(\varphi^i)=&\mathcal{Q}^k_t(\varphi^i)+\epsilon^q_t
\big(R^k_t(\boldsymbol{\gamma}^k,\varphi^i)+\delta
M^k_t(\varphi^i)\big)  \nonumber\\
=& \mathcal{Q}^k_j(\varphi^i)+\sum_{l=j}^t \epsilon^q_l
R^k_l(\boldsymbol{\gamma}^k,\varphi^i)+M^k_t(\varphi^i)
\label{eqn:Q-factor-update-accumulate}
\end{align}
Since $\mathbb{E}_t[M^k_t(\varphi^i)]=M^k_{t-1}(\varphi^i)$,
$M^k_t(\varphi^i)$ is a Martingale sequence. By martingale
inequality, we have $\Pr_j \big\{\sup_{j\leq l\leq t}
|M_l^k(\varphi^i)| \geq \lambda \big\} \leq
\frac{\mathbb{E}_j[|M_t^k(\varphi^i)|^2]}{\lambda^2}$. By the
property of martingale difference noise and the condition on the
stepsize sequence, we have
$\mathbb{E}_j[|M^k_t(\varphi^i)|^2]=\mathbb{E}_j[|\sum_{l=j}^t
\epsilon_l^q \delta
M^k_l(\varphi^i)|^2]=\sum_{l=j}^{t}\mathbb{E}_j[(\epsilon_l^q)^2
(\delta M_l^k(\varphi^i))^2]\leq \bar{M}  \sum_{l=j}^{t}
(\epsilon_l^q)^2< \infty$, where $\bar{M}=\max_{j\leq l\leq t}
{(\delta M_l^k(\varphi^i))^2}<\infty$. Hence, we have $\lim_{j
\rightarrow \infty} \Pr_j \big\{\sup_{j\leq l\leq t}
|M_l^k(\varphi^i)| \geq \lambda \big\}\rightarrow 0$. Thus, from
\eqref{eqn:Q-factor-update-accumulate}, we have
$\mathcal{Q}^k_{t+1}(\varphi^i)=\mathcal{Q}^k_j(\varphi^i)+\sum_{l=j}^t
\epsilon^q_l R^k_l(\boldsymbol{\gamma}^k,\varphi^i)$ a.s. with the
vector form
\begin{align}
\boldsymbol{\mathcal{Q}}^k_{t+1}=
\boldsymbol{\mathcal{Q}}^k_j+\sum_{l=j}^t \epsilon^q_l\mathbf{R}^k_l
 \label{eqn:Q-factor-update-accumulate-vector}
\end{align}
where
$\mathbf{R}^k_l=\mathbf{T}^k(\gamma^k,\boldsymbol{\mathcal{Q}}^k_l)-\boldsymbol{\mathcal{Q}}^k_l-\big(T_r^k(\gamma^k,\boldsymbol{\mathcal{Q}}^k_l)-\mathcal{Q}^k_l(\varphi^r)\big)\mathbf{e}
$ and $\mathbf{e}=[1,\cdots,1]^T$ is the $I_{\varphi}\times 1$ unit
vector.

Next, we shall establish the convergence of the dynamic equation in
\eqref{eqn:Q-factor-update-accumulate-vector} after the martingale
noise are averaged out. Let $\mathbf{g}^k_t$ and $P^k_t$ denote the
reward column vector and the transition probability matrix under the
power allocation $\mathbf{p}^k_t$, which attains the minimum of
$\mathbf{T}^k$ of the $t$-th iteration. Denote
$z^k_t=T_r^k(\gamma^k,\boldsymbol{\mathcal{Q}}^k_t)-\mathcal{Q}^k_t(\varphi^r)$.
Then, we have
\begin{align}
\mathbf{R}^k_t=&\mathbf{g}^k_t+P^k_t\boldsymbol{\mathcal{Q}}^k_t-
\boldsymbol{\mathcal{Q}}^k_t-z^k_t \mathbf{e}\nonumber\\
 \leq&
\mathbf{g}^k_{t-1}+P^k_{t-1}\boldsymbol{\mathcal{Q}}^k_t-
\boldsymbol{\mathcal{Q}}^k_t-z^k_t \mathbf{e} \nonumber\\
\mathbf{R}^k_{t-1}=&\mathbf{g}^k_{t-1}+P^k_{t-1}\boldsymbol{\mathcal{Q}}^k_{t-1}-
\boldsymbol{\mathcal{Q}}^k_{t-1}-z^k_{t-1} \mathbf{e}\nonumber\\
\leq& \mathbf{g}^k_t+P^k_t\boldsymbol{\mathcal{Q}}^k_{t-1}-
\boldsymbol{\mathcal{Q}}^k_{t-1}-z^k_{t-1} \mathbf{e} \nonumber
\end{align}
\begin{align}
\Longrightarrow&
\mathbf{A}^k_{t-1}\mathbf{R}^k_{t-1}-(z^k_t-z^k_{t-1})\mathbf{e}
\leq \mathbf{R}^k_t \nonumber\\
&\leq
\mathbf{B}^k_{t-1}\mathbf{R}^k_{t-1}-(z^k_t-z^k_{t-1})\mathbf{e}, \
\forall k\geq 1 \nonumber\\
 \stackrel{\text{by iterating}}{\Longrightarrow}&
\mathbf{A}^k_{t-1}\cdots\mathbf{A}^k_{t-m}\mathbf{q}^k_{t-m}-(z^k_t-z^k_{t-m})\mathbf{e}
\leq \mathbf{R}^k_t \nonumber\\
&\leq
\mathbf{B}^k_{t-1}\cdots\mathbf{B}^k_{t-m}\mathbf{q}^k_{t-m}-(z^k_t-z^k_{t-m})\mathbf{e}
\nonumber
\end{align}
Since $R^k_t(\gamma^k,\varphi^r)
=T_r^k(\gamma^k,\boldsymbol{\mathcal{Q}}^k_t)-\mathcal{Q}^k_t(\varphi^r)
-\big(T_r^k(\gamma^k,\boldsymbol{\mathcal{Q}}^k_t)-\mathcal{Q}^k_t(\varphi^r)
\big)=0$  $\forall t$, by  \eqref{promatrix-cond-dis}, we have
\begin{align}
&(1-\delta_m) \min_{i'}
R^k_{t-m}(\gamma^k,\varphi^{i'})-(z^k_t-z^k_{t-m})\leq
R^k_t(\gamma^k,\varphi^i) \nonumber\\
&\leq (1-\delta_m) \max_{i'}
R^k_{t-m}(\gamma^k,\varphi^{i'})-(z^k_t-z^k_{t-m}) \forall \,i
\nonumber \\
\Rightarrow& \begin{cases} \min_{i'}R^k_t(\gamma^k,\varphi^{i'})
\geq (1-\delta_m)
\min_{i'} R^k_{t-m}(\gamma^k,\varphi^{i'})\\ \quad \quad \quad \quad \quad \quad \quad \quad \quad -(z^k_t-z^k_{t-m}) &\\
\max_{i'} R^k_t(\gamma^k,\varphi^{i'}) \leq (1-\delta_m) \max_{i'}
R^k_{t-m}(\gamma^k,\varphi^{i'})\\\quad \quad \quad \quad\quad \quad
\quad \quad \quad-(z^k_t-z^k_{t-m}) &
\end{cases} \nonumber\end{align}
\begin{align}
\Rightarrow& \max_{i'} R^k_t(\gamma^k,\varphi^{i'})- \min_{i'}
R^k_t(\gamma^k,\varphi^{i'})\nonumber\\
& \leq (1-\delta_m)\big(\max_{i'} R^k_{t-m}(\gamma^k,\varphi^{i'})-
\min_{i'}
R^k_{t-m}(\gamma^k,\varphi^{i'})\big) \nonumber\\
\Rightarrow& \max_{i'} R^k_t(\gamma^k,\varphi^{i'})- \min_{i'}
R^k_t(\gamma^k,\varphi^{i'}) \leq \phi_j \prod_{l=1}^{\lfloor
\frac{t-j}{m} \rfloor}(1-\delta_{j+lm}) \nonumber
\end{align}
where $\phi_j>0$. Since $R^k_t(\gamma^k,\varphi^r)=0$ $\forall t$,
we have $\max_{i'} R^k_t(\gamma^k,\varphi^{i'})\geq0$ and $\min_{i'}
R^k_t(\gamma^k,\varphi^{i'})\leq 0$. Thus, $\forall \, i$, we have
$|R^k_t(\gamma^k,\varphi^i)| \leq \max_{i'}
R^k_t(\gamma^k,\varphi^{i'})- \min_{i'} R^k_t(\gamma^k,\varphi^{i'})
\leq \phi_j \prod_{l=1}^{\lfloor \frac{t-j}{m}
\rfloor}(1-\delta_{j+lm}) $. Therefore, as $t \rightarrow \infty$,
$\mathbf{R}^k_t \rightarrow
 \mathbf{0}$, i.e. $\boldsymbol{\mathcal{Q}}^k_{\infty}(\boldsymbol{\gamma})$ satisfies equation
in \eqref{bellmaneqn-wp1-dis-fixed-gamma}. Similar to the potential
function of Bellman equation (Proposition 1 in Chapter 7 of
\cite{Bertsekas:2007}), the solution to
\eqref{bellmaneqn-wp1-dis-fixed-gamma} is unique only up an additive
constant. Since
$\mathcal{Q}^k_t(\varphi^r)=\mathcal{Q}^k_0(\varphi^r)$ $\forall t$,
we have the convergence of the per-user subband allocation Q-factor
$\lim_{l\rightarrow
\infty}\boldsymbol{\mathcal{Q}}^k_t=\boldsymbol{\mathcal{Q}}^k_{\infty}(\boldsymbol{\gamma})$
almost surely.

\section*{Appendix C: Proof of Lemma \ref{Lem:convergence-dis-gamma}}

Due to the separation of time scale, the primal update of the
Q-factor can be regarded as converged to
$\boldsymbol{\mathcal{Q}}^k_{\infty}(\boldsymbol{\gamma}_t)$ w.r.t.
the current LMs $\boldsymbol{\gamma}_t$
\cite{Borkartwotimescales:1997}. Using standard stochastic
approximation theorem \cite{Borkarbook:2008}, the dynamics of the
LMs update equation in \eqref{eqn:gamma-dis-est1} and
\eqref{eqn:gamma-dis-est1-under} can be represented by the following
ODE:
\begin{align}
\dot{\boldsymbol{\gamma}}(t)=&\mathbb{E}^{\Omega^*(\boldsymbol{\gamma}(t))}\big[(\sum_n
p_{1,n} - P_1),(\mathbf{1}[Q_k=N_Q] - P^d_1),\cdots,\nonumber\\
& (\sum_n p_{K,n} - P_K), (\mathbf{1}[Q_K=N_Q]-P^d_K)
\big]^T\label{eqn:ode-LM}
\end{align}
where
$\Omega^*(\boldsymbol{\gamma}(t))=\big(\Omega^*_p(\boldsymbol{\gamma}(t)),
\Omega^*_s(\boldsymbol{\gamma}(t))\big)$ is the converged control
policies in \eqref{eqn:general-p} and \eqref{eqn:general-s} w.r.t.
the current LM $\boldsymbol{\gamma}(t)$, and
$\mathbb{E}^{\Omega^*(\boldsymbol{\gamma}(t))}[\cdot]$ denotes the
expectation w.r.t. the measure induced by
$\Omega^*(\boldsymbol{\gamma})$. Define
$G(\boldsymbol{\gamma})=\mathbb{E}^{\Omega^*(\boldsymbol{\gamma})}\big[\sum_k
g_k(\boldsymbol{\gamma}^k,\boldsymbol{\chi}^i_k,\mathbf{s}_k,\mathbf{p}_k)
\big]$. Since subband allocation  policy is discrete, we have
$\Omega_s^*(\boldsymbol{\gamma})=\Omega_s^*(\boldsymbol{\gamma}+\boldsymbol{\delta}_{\gamma})$.
Hence, by chain rule, we have $\frac{\partial G}{\partial \overline
\gamma^k}=\sum_{k,n}\frac{\partial G}{\partial
p^*_{k,n}}\frac{\partial p^*_{k,n}}{\partial\overline \gamma^k}+
\mathbb{E}^{\big(\Omega_p^*(\boldsymbol{\gamma}),\Omega_
s^*(\boldsymbol{\gamma})\big)}[ \sum_n p^*_{k,n} - P_k]$. Since
$\Omega_p^*(\boldsymbol{\gamma}) =\arg \min_{\Omega_{
p}(\boldsymbol{\gamma})}\mathbb{E}^{\big(\Omega_s^*(\boldsymbol{\gamma}),\Omega_{
p}(\boldsymbol{\gamma})\big)}[\sum_k
g_k(\boldsymbol{\gamma}^k,\boldsymbol{\chi}^i_k,\mathbf{s}^*_k,\mathbf{p}_k)]$,
we have $\frac{\partial G}{\partial \overline \gamma^k}=0+
\mathbb{E}^{\big(\Omega_p^*(\boldsymbol{\gamma}),\Omega_{
s}^*(\boldsymbol{\gamma})\big)}\big[ \sum_n p^*_{k,n} -
P_k\big]=\dot{\overline \gamma^k}(t)$. Similarly, $\frac{\partial
G}{\partial \underline \gamma^k}=
\mathbb{E}^{\big(\Omega_p^*(\boldsymbol{\gamma}),\Omega_{
s}^*(\boldsymbol{\gamma})\big)}\big[
\mathbf{1}[Q_k=N_Q]-P^d_k\big]=\dot{\underline \gamma^k}(t)$.
Therefore, we show that the ODE in \eqref{eqn:ode-LM} can be
expressed as $\dot{\boldsymbol \gamma}(t)= \triangledown
G(\boldsymbol{\gamma}(t))$. As a result, the ODE in
\eqref{eqn:ode-LM} will converge to $\triangledown
G(\boldsymbol{\gamma})=0$, which corresponds to \eqref{eqn:tx-pwr2}
and \eqref{eqn:ptdrop2}.

\section*{Appendix D: Proof of Lemma \ref{Lem:reduced-MDP-opt-control-per-user-per-subband-q-factor}}
\begin{align}
\text{Let }q^k(Q_k,|H_{k,n}|,s_{k,n})&=
\min_{p_{k,n}}\big\{g_{k,n}(\boldsymbol{\gamma}^k,Q_k,
|H_{k,n}|,s_{k,n},p_{k,n})\nonumber\\
&
-\frac{\Delta\widetilde{W}^k(Q^i_k)\tau}{\overline{N_k}}s_{k,n}\log(1
+
p_{k,n}|H_{k,n}|^2)\nonumber\\
&+\frac{\mathbb{E}[\widetilde{W}^k(Q^i_k+A_k)|Q_k]}{N_F}
-\frac{\theta^k}{N_F}\big\}\label{eqn:q-factor-define}
\end{align}
where $\widetilde{W}^k(Q_k)\triangleq
\mathbb{E}[W^k(\boldsymbol{\chi}_k)|Q_k]$ and $\Delta
\widetilde{W}^k(Q_k)=\mathbb{E}[\widetilde{W}^k(Q_k+A_k)-\widetilde{W}^k(Q_k+A_k-1)
|Q_k]$. Then, we have
$\mathcal{Q}^k(\boldsymbol{\chi}_k,\mathbf{s}_k)=\sum_n
q^k(Q_k,|H_{k,n}|,s_{k,n})$. Thus, we can derive
\begin{align}
&W^k(\boldsymbol{\chi}_k)=\mathbb{E}\big[
\mathcal{Q}^k\big(\boldsymbol{\chi}_k,\{s_{k,n}
=\mathbf{1}[|H_{k,n}|\geq H^*_{K-1}]\}\big)|\boldsymbol{\chi}_k
\big]\nonumber\\
=&\mathbb{E}\big[\sum_n
q^k(Q_k,|H_{k,n}|,s_{k,n}=\mathbf{1}[|H_{k,n}|\geq H^*_{K-1}])
|\boldsymbol{\chi}_k \big]\nonumber\\
=&\sum_n \underbrace{\mathbb{E}\big[
q^k(Q_k,|H_{k,n}|,s_{k,n}=\mathbf{1}[|H_{k,n}|\geq H^*_{K-1}]) |Q_k,
H_{k,n} \big]}_{w^k(Q_k,|H_{k,n}|)}\nonumber \\
\Rightarrow& \widetilde{W}^k(Q_k)=
\mathbb{E}[W^k(\boldsymbol{\chi}_k)|Q_k]=\mathbb{E}[\sum_n w^k(Q_k,|H_{k,n}| )|Q_k]\nonumber\\
=&\sum_n \underbrace{\mathbb{E}[w^k(Q_k,|H_{k,n}|
)|Q_k]}_{\widetilde{w}^k(Q_k)}=N_F
\widetilde{w}^k(Q_k)\nonumber\\
\Rightarrow &\Delta
\widetilde{W}^k(Q_k)=\mathbb{E}[\widetilde{W}^k(Q_k+A_k)-\widetilde{W}^k(Q_k+A_k-1)|Q_k]\nonumber\\
=&N_F\underbrace{\mathbb{E}[\widetilde{w}^k(Q_k+A_k)-\widetilde{w}^k(Q_k+A_k-1)|Q_k]}_{\delta\widetilde{w}^k(Q_k)}\nonumber
\end{align}
Therefore, from \eqref{eqn:q-factor-define}, we can obtain
\eqref{eqn:per-user-per-subb-q-factor}.

\section*{Appendix E: Proof of Lemma \ref{Lem:per-user-power-allocation}}
The conditional transition probability of user $k$ is given by
$\Pr[\boldsymbol{\chi}^j_k|\boldsymbol{\chi}^i_k,\mathbf{s}_k,\mathbf{p}_k]=\Pr[\mathbf{H}^j_k]\Pr[Q^j_k|\boldsymbol{\chi}^i_k,
\mathbf{s}_k, \mathbf{p}_k]$, where
$\Pr[Q^j_k|\boldsymbol{\chi}^i_k, \mathbf{s}_k,
\mathbf{p}_k]=\Pr[A_k=Q^j_k-Q^i_k+1]\mu_k(\boldsymbol{\chi}^i_k,\mathbf{s}_k,
\mathbf{p}_k) \tau +
\Pr[A_k=Q^j_k-Q^i_k]\big(1-\mu_k(\boldsymbol{\chi}^i_k,\mathbf{s}_k,
\mathbf{p}_k)\tau \big)$.
\begin{align}
&\mathcal{Q}^k(\boldsymbol{\chi}^i_k,\mathbf{s}_k)\nonumber\\
\stackrel{(a)}{=}&\min_{\mathbf{p}_k} \Big[
g_k(\boldsymbol{\gamma}^k,\boldsymbol{\chi}^i_k,\mathbf{s}_k,\mathbf{p}_k)\nonumber\\
& \quad \quad + \sum_{\mathbf{H}^j_k,Q^j_k}
\Pr[\mathbf{H}^j_k]\Pr[Q^j_k|\boldsymbol{\chi}^i_k, \mathbf{s}_k,
\mathbf{p}_k] W^k(\boldsymbol{\chi}^j_k)\Big]-\theta^k \nonumber\\
\stackrel{(b)}{=}&\min_{\mathbf{p}_k} \Big[
g_k(\boldsymbol{\gamma}^k,\boldsymbol{\chi}^i_k,\mathbf{s}_k,\mathbf{p}_k)\nonumber\\
& \quad \quad + \sum_{Q^j_k} \Pr[Q^j_k|\boldsymbol{\chi}^i_k,
\mathbf{s}_k,
\mathbf{p}_k] \widetilde{W}^k(Q^j_k)\Big]-\theta^k \nonumber\\
=&\min_{\mathbf{p}_k} \Big[
g_k(\boldsymbol{\gamma}^k,\boldsymbol{\chi}^i_k,\mathbf{s}_k,\mathbf{p}_k)\nonumber\\
& \quad \quad + \big(1-\mu_k(\boldsymbol{\chi}^i_k,\mathbf{s}_k,
\mathbf{p}_k)\tau \big)\mathbb{E}[\widetilde{W}^k(Q^i_k+A_k)|Q_k]\nonumber\\
& \quad \quad +\mu_k(\boldsymbol{\chi}^i_k,\mathbf{s}_k,
\mathbf{p}_k) \tau
\mathbb{E}[\widetilde{W}^k(Q^i_k+A_k-1)|Q_k]\Big]-\theta_k\nonumber\\
\stackrel{(d)}{\Leftrightarrow}& \min_{\mathbf{p}_k}
\overline{\gamma}^k \sum_n p_{k,n} - \frac{\Delta
\widetilde{W}^k(Q_k)\tau}{\overline{N_k}}\big(\sum_n s_{k,n}\log(1 +
p_{k,n}|H_{k,n}|^2)\big) \label{eqn:max-R.H.S.}
\end{align}
where (a) is due to \eqref{eqn:Bellman-per-user-q-factor} and the
above per-user transition probability, (b) is due to the definition
$\widetilde{W}^k(Q_k)\triangleq
\mathbb{E}[W^k(\boldsymbol{\chi}_k)|Q_k]$ and (d) is due to the
definition $\Delta
\widetilde{W}^k(Q_k)=\mathbb{E}[\widetilde{W}^k(Q_k+A_k)-\widetilde{W}^k(Q_k+A_k-1)
|Q_k]$. By applying standard convex optimization techniques and
Lemma
\ref{Lem:reduced-MDP-opt-control-per-user-per-subband-q-factor}
($\Delta \widetilde{W}^k(Q_k)=N_F\delta\widetilde{w}^k(Q_k)$), the
optimal solution to \eqref{eqn:max-R.H.S.} is given by
\eqref{eqn:joint-pwr-allo}.


\section*{Appendix F: Proof of Lemma \ref{Lem:order-of-growth}}
We first fix $K$ and consider the growth of the ergodic visiting
speed w.r.t. $N$. As $N$ increases, the number of per-user
per-subband state-action pair observations made at each time slot
increases (this``parallelism'' helps to speed up the convergence
rate). Thus, the chance that all per-user per-subband state-action
pair of each user are visited grows like $\mathcal O(N)$, and hence,
the ergodic visiting speed of each user grows like $\mathcal O(N)$.
Next, we fix N and consider the growth of the ergodic visiting speed
w.r.t. $K$. Each subband can only be allocated to one user. Thus,
the chance of the bottleneck state-action pair with $s=1$ for each
user being visited decreases like $\mathcal O(K)$, and hence, the
ergodic visiting speed of each user grows like $\mathcal O(1/K)$.
Combine the above two cases, we conclude Lemma
\ref{Lem:order-of-growth}.

\section*{Appendix G: Proof of Theorem \ref{Thm:asymp-performance-dis-alg}}

For given $\boldsymbol{\gamma}$, we shall prove that under a {\em
Best-CSI subband allocation policy}, the Q-factor satisfying the
Bellman equation \eqref{eqn:Bellman-q-factor} can be decomposed into
the additive form in \eqref{eqn:approximate-Q}. Based on that, we
shall show that for large $K$, the linear Q-factor approximation in
\eqref{eqn:approximate-Q} is indeed optimal.



\begin{Def}\label{def semi-global}[Best-CSI Subband Allocation Policy] A {\em Best-CSI subband allocation policy} is defined as  $\widetilde\Omega_s(\mathbf{H})=\{\tilde{s}_{k,n}(\mathbf{H}_n)\in\{0,1\}|\sum_{k=1}^{K}\tilde{s}_{k,n}=1 \,
\forall n\}$, where
\begin{align}
\widetilde{s}_{k,n}(\mathbf{H}_n)=&\mathbf{1}[|H_{k,n}|=
\max_j|H_{j,n}|]\nonumber\\
=&\mathbf{1}[|H_{k,n}|\geq \max_{j\neq
k}|H_{j,n}|]\label{eqn:best-H-subband}
\end{align}
\end{Def}
We first establish a property of the Q-factor in the original
Bellman equation in \eqref{eqn:Bellman-q-factor} under the {\em
Best-CSI subband allocation policy}, which is summarized in Lemma
\ref{Lem:additive-property}.

\begin{Lem}(\emph{Additive Property of the Subband Allocation Q-Factor})
Under the {\em Best-CSI subband allocation policy}, the solution to
the original Bellman equation in \eqref{eqn:Bellman-q-factor} can be
expressed into the form
$\mathcal{Q}(\boldsymbol{\chi},\mathbf{s})=\sum_k
\mathcal{Q}^k_{\infty}(\boldsymbol{\chi}_k,\mathbf{s}_k)$, where
$\{\mathcal{Q}^k_{\infty}(\boldsymbol{\chi}_k,\mathbf{s}_k)\}$ is
the converged per-user Q-factor, which is also the  solution of the
$k$-th user's {\em per-user subband allocation Q-factor fixed point
equation} given by \eqref{eqn:Bellman-per-user-q-factor}.
\label{Lem:additive-property}
\end{Lem}
\begin{proof}
Under the {\em Best-CSI subband allocation policy}, the Bellman
equation  in \eqref{eqn:Bellman-q-factor}  becomes
\begin{align}
&\mathcal{Q}(\boldsymbol{\chi}^i,\mathbf{s})
\stackrel{(a)}{=}\min_{\Omega_p(\boldsymbol{\chi}^i)} \Big[
g(\boldsymbol{\gamma},\boldsymbol{\chi}^i,\mathbf{s},\Omega_p(\boldsymbol{\chi}^i))\nonumber\\
&+\sum_{\mathbf{Q}^j}
\Pr[\mathbf{Q}^j|\boldsymbol{\chi}^i,\mathbf{s},\Omega_p(\boldsymbol{\chi}^i)]\underbrace{\big(\sum_{\mathbf{H}^j}\Pr[\mathbf{H}^j]\mathcal{Q}(\boldsymbol{\chi}^j,\widetilde\Omega_s(\mathbf{H}^j))\big)}_{\widetilde{V}(\mathbf{Q}^j)}\Big]-\theta\nonumber\\
&\quad \quad \quad \quad \quad \quad \quad \quad \quad \forall
\,1\leq i\leq I_{\chi}, \ \forall \mathbf{s}\label{eqn:additive-1}
\end{align}
\begin{align}
\stackrel{(b)}{\Leftrightarrow}&
\widetilde{V}(\mathbf{Q}^i)=\sum_{\mathbf{H}^i}\Pr[\mathbf{H}^i]\min_{\Omega_p(\boldsymbol{\chi}^i)}
\Big[
g(\boldsymbol{\gamma},\boldsymbol{\chi}^i,\widetilde\Omega_s(\mathbf{H}^i),\Omega_p(\boldsymbol{\chi}^i))\nonumber\\
&+ \sum_{\mathbf{Q}^j}
\Pr[\mathbf{Q}^j|\boldsymbol{\chi}^i,\widetilde\Omega_s(\mathbf{H}^i),\Omega_p(\boldsymbol{\chi}^i)]\widetilde{V}(\mathbf{Q}^j)\Big]-\theta,
\ 1\leq i\leq I_{Q}\label{eqn:reduced-q-factor}
\end{align}
where (a) is due to \eqref{eqn:transition-prob1} and the definition
$\widetilde{V}(\mathbf{Q}) \triangleq
\mathbb{E}[\mathcal{Q}(\boldsymbol{\chi},\widetilde\Omega_s(\mathbf{H}))|\mathbf{Q}]$,
(b) is obtained by taking conditional expectation (conditioned on
$\mathbf{Q}^i$) on both sides of \eqref{eqn:additive-1} and the
definition of $\widetilde{V}(\mathbf{Q})$. In addition, denote
$\Delta_k
\widetilde{V}(\mathbf{Q})\triangleq\mathbb{E}[\widetilde{V}(\mathbf{Q}+\mathbf{A})-\widetilde{V}(\mathbf{Q}+\mathbf{A}-\mathbf{e}_k)|\mathbf{Q}]$.

From \eqref{eqn:additive-1}, we know that
$\{\mathcal{Q}(\boldsymbol{\chi}^i,\mathbf{s})\}$ is determined by
$\{\widetilde{V}(\mathbf{Q}^i)\}$. Next, we shall try to solve
$\{\widetilde{V}(\mathbf{Q}^i)\}$ by the $I_Q$ equations in
\eqref{eqn:reduced-q-factor}. First, assume the linear approximation
$\mathcal{Q}(\boldsymbol{\chi},\widetilde\Omega_s(\mathbf{H}))=
\sum_k \mathcal{Q}^k(\boldsymbol{\chi}_k,
\widetilde\Omega_s^k(\mathbf{H}))$ holds under the best-CSI subband
allocation policy, we have 
\begin{align}
&\widetilde{V}(\mathbf{Q}) = \mathbb{E}[\sum_k
\mathcal{Q}^k(\boldsymbol{\chi}_k,\widetilde\Omega_s^k(\mathbf{H}))|\mathbf{Q}]\nonumber\\
=&\sum_k \mathbb{E}[
\mathcal{Q}^k(Q_k,\mathbf{H}_k,\widetilde\Omega_s^k(\mathbf{H}))|\mathbf{Q}]\nonumber\\
=&\sum_k \mathbb{E}[
\mathcal{Q}^k(Q_k,\mathbf{H}_k,\widetilde\Omega_s^k(\mathbf{H}))|Q_k]\nonumber\\
=& \sum_k \mathbb{E}\Big[\mathbb{E}\big[
\mathcal{Q}^k\big(Q_k,\mathbf{H}_k,\{\widetilde{s}_{k,n}=\mathbf{1}[|H_{k,n}|\geq
\max_{j\neq k}|H_{j,n}|]\}\big)\nonumber\\
&\quad \quad \quad |Q_k,\mathbf{H}_k\big]|Q_k \Big]=\sum_k
\mathbb{E}[W^k(\boldsymbol{\chi}_k)|Q_k]=\sum_k
\widetilde{W}^k(Q_k)\nonumber
\end{align}
\begin{align}
\Delta_k \widetilde{V}(\mathbf{Q})=&\mathbb{E}[\sum_j
\widetilde{W}^j(Q_j+A_j)-\big(\sum_{j\neq k}
\widetilde{W}^j(Q_j+A_j)\nonumber\\
&+\widetilde{W} ^k(Q_k+A_k-1)\big) |\mathbf{Q}]=\Delta
\widetilde{W}^k(Q_k)\nonumber
\end{align}
Thus, the optimal power allocation and corresponding conditional
departure rate to $\min_{\Omega_p(\boldsymbol{\chi}^i)} [\cdot ]$
part in \eqref{eqn:reduced-q-factor} are as follows
\begin{align}
&p_{k,n}(Q_k,|H_{k,n}|,\widetilde{s}_{k,n}(\mathbf{H}_n)),\ \forall k,n\label{eqn:best-H-opt-power-Qk}\\
=&\widetilde{s}_{k,n}(\mathbf{H}_n)\Big(\frac{\frac{\tau}{\overline{N}_k}
\Delta
\widetilde{W}^k(Q_k)}{\overline{\gamma}^k}-\frac{1}{|H_{k,n}|^2}\Big)^{+}\nonumber
\\
&\mu_k(Q_k,\mathbf{H}_k,\widetilde{\mathbf{s}}_k(\mathbf{H})),\ \forall k\label{eqn:best-H-opt-depart-rate-Qk}\\
=&\frac{1}{\overline{N_k}}\sum_n
\widetilde{s}_{k,n}(\mathbf{H}_n)\log(1+ p_{k,n}(Q_k,
|H_{k,n}|,\widetilde{s}_{k,n}(\mathbf{H}_n))|H_{k,n}|^2)\nonumber
\end{align}
Therefore, from \eqref{eqn:reduced-q-factor}, we have
\begin{align}
&\sum_k \widetilde{W}^k(Q^i_k)=\sum_k \Big(
\widetilde{g}_k(\boldsymbol{\gamma}^k,Q^i_k)+\mathbb{E}[\widetilde{W}^k(Q^i_k+A_k)|Q^i_k]\nonumber\\
& \quad \quad \quad \quad \quad \quad \quad -
\widetilde{\mu}_k(Q^i_k)\tau \Delta \widetilde{W}^k(Q^i_k)
\Big)-\theta\nonumber\\
\Rightarrow& \theta=\sum_k \theta^k =\sum_k\Big(
\widetilde{g}_k(\boldsymbol{\gamma}^k,Q^i_k)+\mathbb{E}[\widetilde{W}^k(Q^i_k+A_k)|Q^i_k]\nonumber\\
&\quad \quad \quad \quad \quad \quad  -
\widetilde{\mu}_k(Q^i_k)\tau \Delta
\widetilde{W}^k(Q^i_k)- \widetilde{W}^k(Q^i_k) \Big),\nonumber\\
& \quad \quad \quad \quad \quad \quad \quad \quad  \quad \quad \quad
\quad 1\leq i\leq I_{Q}\label{eqn:decoupled-structure}
\end{align}
where $\widetilde{g}_k(\boldsymbol{\gamma}^k,Q^i_k)=\mathbb{E}\big[
\beta_k f(Q_k)+\overline{\gamma}^k (\sum_n p_{k,n}(Q_k,
|H_{k,n}|,\widetilde{s}_{k,n}(\mathbf{H}_n))-P_k)
+\underline{\gamma}^k(\mathbf{1}[Q^i_k=N_Q]-P_k^d)|Q_k\big]$ and
$\widetilde{\mu}_k(Q_k)=\mathbb{E}\big[\mu_k(Q_k,\mathbf{H}_k,\widetilde{\mathbf{s}}_k(\mathbf{H}))|Q_k\big]$.
Since there are only $(N_Q+1)$ QSI states for each user and the
structure in \eqref{eqn:decoupled-structure} is decoupled  under the
additive assumption, for each user $k$, there are only $(N_Q+1)$
independent Poisson equations with $N_Q+2$ unknowns $\{\theta^k,
\widetilde{W}^k(Q_k)\}$. $\theta_k$ is unique and
$\{\widetilde{W}^k(Q_k)\}$ is unique up to an additive
constant\cite{Bertsekas:2007}. Therefore, $\{\theta,
\widetilde{V}(\mathbf{Q})\}$ is the solution to
\eqref{eqn:reduced-q-factor}, where $\theta=\sum_k \theta^k$ and
$\widetilde{V}(\mathbf{Q})=\sum_k\widetilde{W}^k(Q_k) $.

Next, we shall show
$\mathcal{Q}(\boldsymbol{\chi},\mathbf{s})=\sum_k
\mathcal{Q}^k_{\infty}(\boldsymbol{\chi}_k,\mathbf{s}_k)$.
Substitute $\theta=\sum_k \theta^k$ and
$\widetilde{V}(\mathbf{Q})=\sum_k\widetilde{W}^k(Q_k) $ into
\eqref{eqn:additive-1}, we have
\begin{align}
\mathcal{Q}(\boldsymbol{\chi}^i,\mathbf{s})
=&\min_{\Omega_p(\boldsymbol{\chi}^i)} \Big[
g(\boldsymbol{\gamma},\boldsymbol{\chi}^i,\mathbf{s},\Omega_p(\boldsymbol{\chi}^i)) \nonumber \\
&+\sum_{\mathbf{Q}^j}
\Pr[\mathbf{Q}^j|\boldsymbol{\chi}^i,\mathbf{s},\Omega_p(\boldsymbol{\chi}^i)]\big(\sum_k\widetilde{W}^k(Q^j_k)\big)\Big]-\sum_k
\theta^k\nonumber\\
=&\sum_k \mathcal{Q}^k(\boldsymbol{\chi}^i_k,\mathbf{s}_k)\nonumber
\end{align}
where
$\mathcal{Q}^k(\boldsymbol{\chi}^i_k,\mathbf{s}_k)=\min_{\mathbf{p}_k}
\Big[
g_k(\boldsymbol{\gamma}^k,\boldsymbol{\chi}^i_k,\mathbf{s}_k,\mathbf{p}_k)
+ \sum_{Q^j_k}
\Pr[Q^j_k|\boldsymbol{\chi}^i_k,\mathbf{s}_k,\mathbf{p}_k]
\widetilde{W}^k(Q^j_k)\Big]-\theta^k$, which is equivalent to
\eqref{eqn:Bellman-per-user-q-factor}. By Lemma
\ref{Lem:convergence-dis-fixed-gamma}, the converged
$\{\mathcal{Q}^k_{\infty}(\boldsymbol{\chi}_k,\mathbf{s}_k)\}$
satisfy  \eqref{eqn:Bellman-per-user-q-factor} and this completes
the proof.
\end{proof}

Next, we shall consider the asymptotic subband allocation results
for large $K$. The optimal control actions to
\eqref{eqn:Bellman-q-factor} are given by
\begin{align}
p_{k,n}(\mathbf{H}_n, \mathbf{Q})=&s_{k,n}(\mathbf{H}_n,
\mathbf{Q})\Big(\frac{\frac{\tau}{\overline{N}_k}
\Delta_k \widetilde{V}^*(\mathbf{Q})}{\overline{\gamma}^k}-\frac{1}{|H_{k,n}|^2}\Big)^{+}\label{eqn:joint-pwr-allo-c}\\
s_{k,n}(\mathbf{H}_n, \mathbf{Q}) =& \left\{
\begin{array}{ll} 1, &
\text{if} \quad X_{k,n}=\max_{j} \big\{X_{j,n}\big\} >0\\
0, & \textrm{otherwise}
\end{array} \right. \label{eqn:joint-subc-allo}
\end{align}
where $\widetilde{V}^*(\mathbf{Q}) \triangleq
\mathbb{E}[\min_{\mathbf{s}}
\mathcal{Q}^*(\boldsymbol{\chi},\mathbf{s})|\mathbf{Q}]$,
$\Delta_k\widetilde{V}^*(\mathbf{Q})\triangleq\mathbb{E}[\widetilde{V}^*(\mathbf{Q}+\mathbf{A})-\widetilde{
V}^*(\mathbf{Q}+\mathbf{A}-\mathbf{e}_k)|\mathbf{Q}]$ and
$X_{k,n}=\frac{\tau}{\overline{N}_k}\Delta_k
\widetilde{V}^*(\mathbf{Q})
\log\Big(1+|H_{k,n}|^2\big(\frac{\frac{\tau}{\overline{N}_k}\Delta_k
\widetilde{V}^*(\mathbf{Q})}{\overline{\gamma}^k}-\frac{1}{|H_{k,n}|^2}\big)^{+}\Big)-\overline{\gamma}^k\big(\frac{\frac{\tau}{\overline{N}_k}\Delta_k
\widetilde{V}^*(\mathbf{Q})}{\overline{\gamma}^k}-\frac{1}{|H_{k,n}|^2}\big)^{+}$.
Denote $k^*_n\triangleq \arg \max_k |H_{k,n}|^2$. For large $K$,
$|H_{k^*_n,n}|^2$ grows with $\log(K)$  by extreme value theory.
Because the traffic loading remains unchanged as we scale up $K$,
$\max_{k,j}|\Delta_k
\widetilde{V}^*(\mathbf{Q})-\Delta_j\widetilde{V}^*(\mathbf{Q})| =
O(1)$. Hence, $X_{k^*_n,n}$ grows like $\log(\log(K))$. As
$K\rightarrow \infty$, $\Pr[k^*_n =\arg \max_k X_{k,n}]=1$. Thus the
subband allocation result of optimal subband allocation in
\eqref{eqn:joint-subc-allo} and the best CSI subband allocation in
\eqref{eqn:best-H-subband} will be the same for large $K$. Using the
result in Lemma \ref{Lem:additive-property}, the linear Q-factor
approximation is therefore asymptotically accurate for given
$\boldsymbol{\gamma}$.
Combining the results of Theorem \ref{thm:convergence-dis}, we can
prove Theorem \ref{Thm:asymp-performance-dis-alg}.

\end{appendix}


\bibliographystyle{IEEEtran}
\bibliography{IEEEabrv,delay-sensitive-OFDMA-TW}


\begin{biography}[{\includegraphics[width=1.0in,height=1.1in,clip,keepaspectratio]{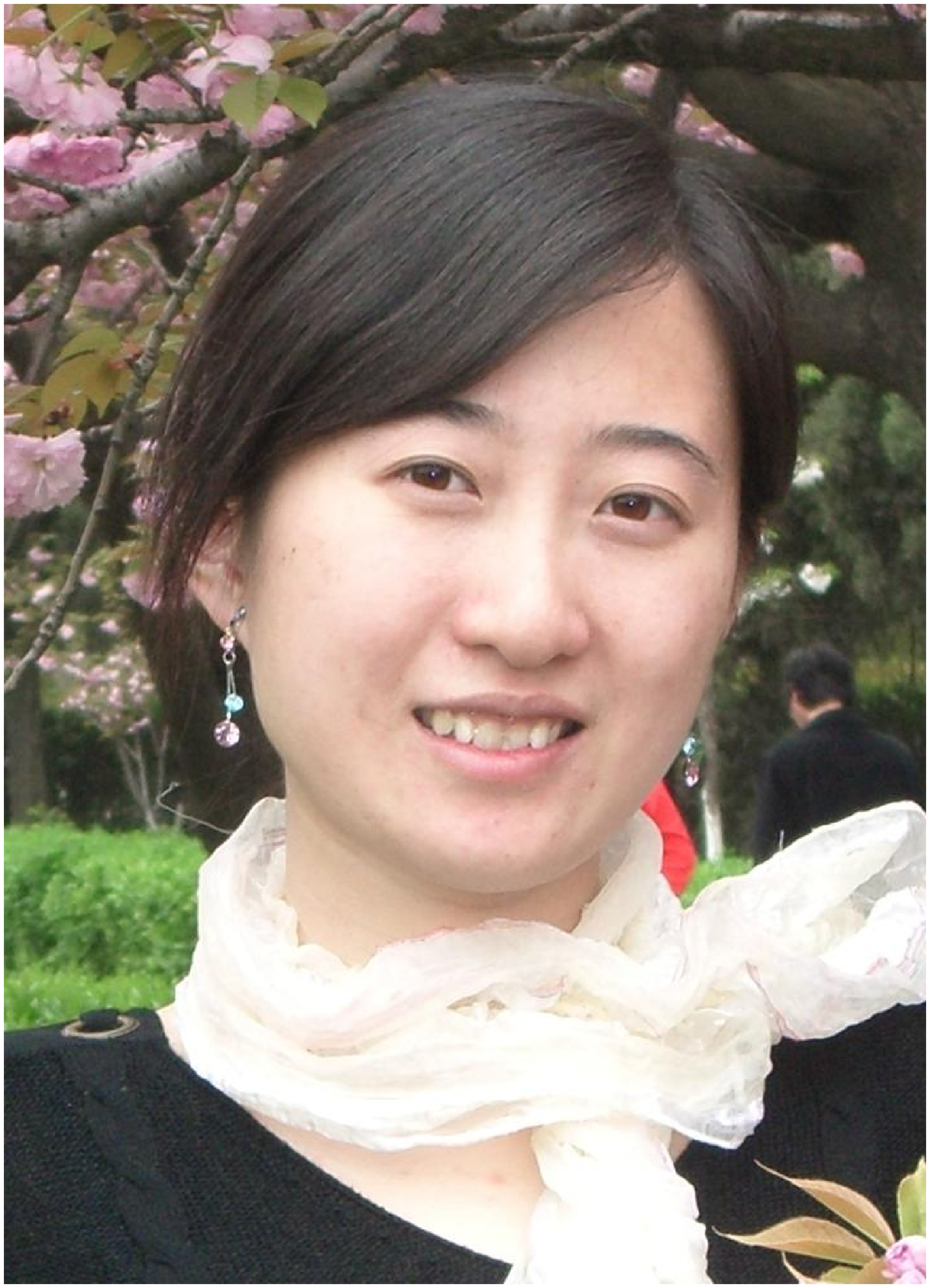}}]{Ying~Cui}
received B.Eng degree (first class honor) in Electronic and
Information Engineering, Xi¡¯an Jiaotong University, China in 2007.
She is currently a Ph.D candidate in the Department of ECE, the Hong
Kong University of Science and Technology (HKUST). Her current
research interests include cooperative and cognitive communications,
delay-sensitive cross-layer scheduling as well as stochastic
approximation and Markov Decision Process.
\end{biography}
\begin{biography}[{\includegraphics[width=1.0in,height=1.1in,clip,keepaspectratio]{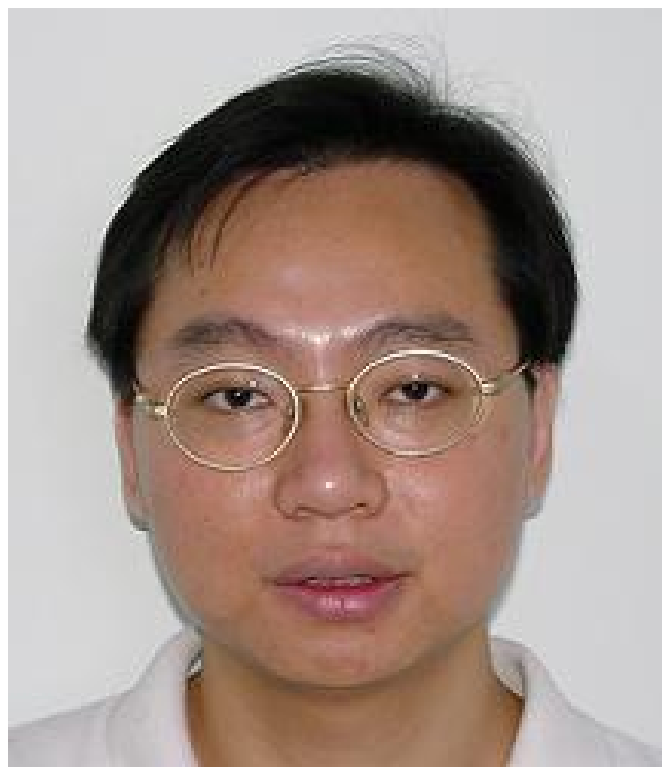}}]{Vincent~K.~N.~Lau}
obtained B.Eng (Distinction 1st Hons) from the University of Hong
Kong in 1992 and Ph.D. from Cambridge University in 1997. He was
with PCCW as system engineer from 1992-1995 and Bell Labs - Lucent
Technologies as member of technical staff from 1997-2003. He then
joined the Department of ECE, HKUST as Associate Professor. His
current research interests include the robust and delay-sensitive
cross-layer scheduling, cooperative and cognitive communications as
well as stochastic approximation and Markov Decision Process.
\end{biography}

\end{document}